\documentclass[mnsc]{informs3} 

\OneAndAHalfSpacedXI 

\usepackage{natbib}
 \bibpunct[, ]{(}{)}{,}{a}{}{,}%
 \def\bibfont{\small}%

\usepackage{booktabs}
\usepackage{xcolor}

\newcommand{\pdr}{\mathrm{PDR}}

\TheoremsNumberedThrough     
\ECRepeatTheorems

\EquationsNumberedThrough    

\MANUSCRIPTNO{MS-000-0000}

\usepackage[colorlinks=true,
            linkcolor=blue,
            citecolor=blue,
            urlcolor=blue, hypertexnames=false]{hyperref}

\makeatletter
\g@addto@macro\UrlBreaks{\do\a\do\b\do\c\do\d\do\e\do\f\do\g\do\h\do\i\do\j
  \do\k\do\l\do\m\do\n\do\o\do\p\do\q\do\r\do\s\do\t\do\u\do\v\do\w\do\x\do\y\do\z
  \do\A\do\B\do\C\do\D\do\E\do\F\do\G\do\H\do\I\do\J\do\K\do\L\do\M\do\N\do\O
  \do\P\do\Q\do\R\do\S\do\T\do\U\do\V\do\W\do\X\do\Y\do\Z\do\0\do\1\do\2\do\3
  \do\4\do\5\do\6\do\7\do\8\do\9}
\makeatother
\usepackage{bookmark}
\bookmarksetup{depth=2}

\makeatletter

\def\theARTICLETOP{}

\def\ps@headings{%
  \def\@oddhead{}%
  \def\@evenhead{}%
  \def\@oddfoot{\hfil\thepage\hfil}%
  \def\@evenfoot{\hfil\thepage\hfil}%
}

\def\ps@ECheadings{%
  \def\@oddhead{}%
  \def\@evenhead{}%
  \def\@oddfoot{\hfil\thepage\hfil}%
  \def\@evenfoot{\hfil\thepage\hfil}%
}

\makeatother

\AtBeginDocument{\pagestyle{headings}}

\begin{document}



\RUNTITLE{Counterfactual Optimal Action Trees (COAT)}

\TITLE{Counterfactual Optimal Action Trees (COAT): Interpretable Prescriptive Policies from \\Observational Data}


\ARTICLEAUTHORS{%
\AUTHOR{Youssef Drissi \quad Markus Ettl \quad
Shivaram Subramanian$^\dagger$ \quad
Wei Sun$^\dagger$ \quad Zack Xue}
\vspace{8pt}
\AFF{\small IBM Research}
\footnotetext{$^\dagger$Equal contributors.}
}

\ABSTRACT{%
We introduce \textsc{COAT} (Counterfactual Optimal Action Tree), a framework for learning
interpretable prescriptive policies from observational data. \textsc{COAT} combines
counterfactual outcome estimation with large-scale mixed-integer optimization, using column generation to
translate causal predictions into feasible, transparent decisions under business
and regulatory constraints. We apply \textsc{COAT} to airline ancillary pricing, a setting characterized by
complex business rules and limited experimental flexibility. In a 17-week field
pilot with a major global airline, \textsc{COAT} increased upsell revenue per
booking by 6.9\%, with the airline projecting \$50--\$150 million in incremental
annual premium seat revenue across eligible domestic markets.
The success of the pilot led to scaled
adoption and informed broader AI-driven decision initiatives within the
organization.
}%




\maketitle

%


\section{Introduction}

Artificial intelligence (AI) is increasingly embedded in high-stakes business decisions across industries. In retail, it supports personalized recommendations and dynamic offers \citep{yang2022dynamic,cao2019doubly}; in financial services, credit risk assessment and fraud detection \citep{goyal2025adoption,fritz2022financial}; and in logistics and supply chains, demand forecasting and routing decisions \citep{khedr2024enhancing}. Across these domains, AI promises faster adaptation, finer-grained personalization, and improved operational efficiency.

Pricing is among the most visible and consequential applications of this transformation, and airlines provide a particularly compelling setting. Airlines were early adopters of revenue management (RM) systems that forecast demand, allocate capacity across fare classes, and dynamically open or close price buckets. While these systems represented a major advance over static pricing, they remain constrained by human-designed structures and relatively simple statistical models. Recent advances in AI raise the prospect of moving beyond fare buckets entirely, learning pricing policies directly from large-scale data and adapting them dynamically to changing conditions.

This prospect has begun to materialize in practice. Delta Air Lines disclosed the deployment of AI-based pricing tools on a subset of domestic routes, reporting strong revenue gains and announcing plans to expand AI pricing from roughly 3\% to 20\% of fares by the end of 2025 \citep{Fortune2025,ABCNews2025}. According to company executives, these systems process thousands of signals, such as booking patterns, competitor fares, and schedule changes, far more rapidly than traditional analyst-driven workflows. At the same time, the announcement triggered public and regulatory scrutiny, with concerns raised about fairness, transparency, and compliance. Delta emphasized that its approach does not rely on personal data and operates within existing fare filing systems \citep{DeltaNewsHub2025,Reuters2025}.

This episode illustrates a broader challenge in enterprise AI adoption. While modern AI systems excel at prediction and scale, deploying them responsibly in high-stakes decision environments requires more than predictive accuracy. Decisions must be interpretable, auditable, and consistent with regulatory, fairness, and operational constraints. In many applications, these requirements are binding constraints that determine whether an AI system can be deployed at all. Bridging the gap between predictive AI and deployable decision-making remains a central challenge. 

To address this challenge, we introduce the \emph{Counterfactual Optimal Action Tree} (\textsc{COAT}) framework, a prescriptive decision-making approach that learns interpretable policies directly from observational data. \textsc{COAT} integrates counterfactual outcome estimation with large-scale mixed-integer optimization to construct policies that explicitly respect operational constraints. The framework proceeds in two stages: estimating counterfactual outcomes for alternative actions, and solving a constrained optimization problem to produce an interpretable policy in the form of an action tree. To ensure scalability, we develop a path-based formulation solved via column generation, enabling optimization over large policy spaces. Our central contribution is the design and field validation of this framework in a live, regulated deployment; the accompanying theoretical results characterize when the path-based formulation remains tractable and interpretable at deployment scale, and serve that end rather than standing alone.

We validate \textsc{COAT} in collaboration with a major international airline, focusing on premium seat pricing as a high-value ancillary revenue stream. In a 17-week live pilot, \textsc{COAT} achieved a 6.9\% uplift in upsell revenue per booking, which the airline projected would translate to \$50--\$150 million in incremental annual premium seat revenue across eligible domestic markets. The impact was evaluated using the synthetic control method, providing credible causal evidence in a setting where randomized experiments were infeasible. Following the pilot, the airline scaled adoption of \textsc{COAT} and initiated broader enterprise efforts around AI-enabled decision systems.

More broadly, this work highlights the role of operations research (OR) in the era of AI. While advances in machine learning have dramatically expanded predictive capabilities, large-scale deployments that directly optimize decisions under operational constraints remain rare. By integrating causal estimation, scalable optimization, and interpretable policy design, \textsc{COAT} demonstrates how OR can serve as a decision layer that bridges predictive AI and real-world deployment. The results suggest that the practical value of AI lies not only in better prediction, but in better-designed decision systems that align performance with organizational and societal requirements.

\section{Literature Review}

This paper connects several streams of research, including explainable AI,
tree-based prescriptive optimization, causal decision making, and airline
pricing. While each literature addresses important aspects of decision support,
none alone resolves the joint challenges of learning from observational data,
optimizing decisions under operational constraints, and producing interpretable
policies suitable for deployment. We review the most closely related work and
position \textsc{COAT} relative to these streams.

\subsection{Explainable AI (XAI)}

Interpretability has become a central concern in artificial intelligence as
complex models are increasingly deployed in high-stakes business and policy
settings \citep{allen2023interpretable,arrieta2020explainable,du2019techniques}.
Beyond trust and transparency, interpretability is often a prerequisite for
regulatory compliance, auditing, and organizational adoption.

Existing work broadly falls into two categories. \emph{Intrinsic
interpretability} approaches embed transparency directly into the model
structure, including decision trees and rule lists
\citep{lakkaraju2016interpretable}, generalized additive models and their
variants \citep{rudin2019stop,caruana2015intelligible}, as well as prototype-based
and self-explaining architectures
\citep{kim2016examples,li2018deep,alvarez2018towards}. These models aim to avoid
black-box behavior altogether.

A complementary line of research develops \emph{post-hoc explanation} methods
that rationalize the behavior of trained black-box models. Prominent examples
include local approximation and attribution techniques such as LIME and SHAP
\citep{ribeiro2016should,lundberg2017unified}, as well as counterfactual
explanations \citep{tsiourvas2024manifold} and knowledge distillation, in which
complex models are summarized using simpler, interpretable surrogates
\citep{hinton2015distilling,puiutta2020explainable}. 

Our work is closest in spirit to this latter line, but differs in two important
ways. First, we focus on the \emph{prescriptive} setting, where the objective is
to recommend optimal actions rather than explain predictions. Second, the
interpretable policy in \textsc{COAT} is not a surrogate approximation of a
black-box decision rule. Instead, it is learned directly by optimizing over
counterfactual outcomes produced by a causal model. This separation allows
\textsc{COAT} to combine the expressive power of modern predictive models with
the transparency and auditability of tree-based decision policies.

\subsection{Tree-Based Prescriptive Methods}

Decision trees have long been valued for their interpretability and rule-based
structure \citep{NEURIPS2020_1373b284}. Because learning optimal trees is
NP-complete \citep{laurent1976constructing}, early methods such as CART relied on
greedy heuristics \citep{breiman1984classification}. More recent advances learn
globally optimal trees, either via mixed-integer programming (MIP)
\citep{aghaei2019learning,aghaei2020learning,bertsimas2017optimal} or via
specialized dynamic-programming and branch-and-bound algorithms with analytical
bounds, which scale to substantially larger instances by exploiting the recursive
structure of the tree
\citep{hu2019optimal,lin2020generalized,van2022fair}.

While this literature originated in prediction, similar ideas have been extended
to prescriptive settings, where the goal is to recommend actions rather than
predict outcomes. Prescriptive (or policy) trees have been learned using greedy
approaches \citep{biggs2020model,kallus2017recursive} and optimal formulations
\citep{amram2020optimal,bertsimas2019optimal,ikeda2023prescriptive,jo2021learning,pacepoetree}. Related work has also explored neural network--based policies that can be
post-hoc converted into prescriptive trees \citep{sun2023learning}.

The core algorithmic idea underlying \textsc{COAT} first appeared in an anonymized conference publication \citep{AnonymousAAAI2022}, which introduced the path-based formulation and column-generation algorithm and reported computational comparisons against arc-based and other prescriptive-tree methods on synthetic and publicly available real datasets. The present paper substantially extends that foundation by developing theoretical results that characterize the formulation's behavior under deployment constraints, and by addressing the design, scalability, and live deployment of prescriptive policies in a real-world operational setting.

Our approach departs from prior prescriptive tree methods in its underlying
optimization formulation. Existing MIP-based approaches typically adopt an
arc-based representation, introducing decision variables for each possible branch
at each depth; the resulting formulations grow exponentially with tree depth,
creating scalability limitations \citep{bertsimas2017optimal}. Specialized
dynamic-programming and branch-and-bound methods
\citep{hu2019optimal,lin2020generalized,van2022fair} substantially improve the
scalability of optimal-tree learning, but they are developed primarily for the
\emph{predictive} setting and do not readily accommodate the operational and
inter-rule constraints that are central to prescriptive deployment.

In contrast, \textsc{COAT} adopts a path-based formulation that treats each
root-to-leaf path as a decision unit. This representation enables solution via
column generation \citep{bach2008exploring,jawanpuria2011efficient}, dynamically
generating only the most promising paths. Beyond computational efficiency, the
path-based structure naturally supports rich operational, regulatory, and
fairness constraints defined over complete decision rules. It also yields
multiway-split trees that are typically more compact and interpretable than their
binary-split counterparts, as each feature appears at most once along any
root-to-leaf path.

\subsection{Causal Analysis}

This paper relates to the emerging literature on \emph{causal decision making},
which integrates causal modeling with optimization to recommend actions
\citep{fernandez2022causal,tsiourvas2025causal}. This perspective differs from classical causal
inference, which primarily focuses on estimating treatment effects
\citep{athey2016recursive}, rather than mapping those estimates into optimal
decisions under constraints. Recent work by \citet{harsha2025practical} combines counterfactual cost estimation with column generation in a contextual bandit framework to prescribe resource allocation decisions under service-level constraints.

In our setting, causal modeling provides counterfactual estimates of purchase
behavior under alternative prices. However, such estimates alone do not yield
feasible or optimal policies. \textsc{COAT} bridges this gap by embedding
counterfactual predictions within a large-scale optimization framework that
explicitly incorporates operational and business constraints, resulting in
interpretable prescriptive policies.

Causal inference also plays a central role in evaluation. In many high-stakes
applications, randomized A/B testing is infeasible, as was the case in our
airline deployment. To address this, we employ the synthetic control method (SCM)
\citep{abadie2010synthetic}, which constructs a weighted combination of control
units to replicate pre-intervention outcomes of treated units.  
We refer readers to \citet{abadie2021using} for a comprehensive review.

\subsection{Airline Pricing and Pricing Field Experiments}

Pricing has long been a central application of operations research and management
science, with airlines serving as one of its most prominent success stories
\citep{talluri2006theory,phillips2021pricing}. Classical revenue management
systems focus on capacity control, allocating limited inventory across
predefined fare classes rather than directly setting prices.

Recent work has increasingly explored dynamic and data-driven pricing approaches
\citep{shukla2019dynamic,fiig2018dynamic,mcafee2006dynamic}. Our work differs by
focusing on ancillary products and directly optimizing prices using
counterfactual demand estimates derived from observational data. In addition to
methodological contributions, we demonstrate the approach in a live airline
deployment, providing rare field-based evidence in a highly regulated pricing
environment.

More broadly, our study contributes to the growing literature on pricing field
experiments and evaluations, including large-scale deployments in e-commerce and
retail \citep{liu2019dynamic,cooprider2023science,chen2020high,caro2012clearance,
gao2022socialize}. By conducting a multi-month live pilot in airline ancillary
pricing, we extend this line of work to a domain characterized by complex
operational constraints, regulatory scrutiny, and strong fairness
considerations.

\section{Problem Setup and Prescriptive Framework}

This section formalizes the causal decision-making problem and introduces the
notation used throughout. We describe the observational data setting and causal
assumptions needed to identify counterfactual outcomes, and then present the
\textsc{COAT} framework, which decouples counterfactual estimation from policy
optimization over interpretable policy trees. This formulation underpins the
mixed-integer programming model and the algorithm developed in
subsequent sections.

\subsection{Setup and Notation}

We consider a decision-maker who selects an action (or treatment) from a discrete
set \( [m] := \{1, \dots, m\} \) for each subject characterized by covariates
\( X \in \mathbb{R}^k \). Each subject has potential outcomes
\( \{Y(1), \dots, Y(m)\} \), where \( Y(t) \) denotes the outcome that would be
observed if action \( t \) were assigned. The outcome \( Y \) may be discrete or
continuous. For example, in the pricing application in
Section~\ref{sect_case_study}, \( Y \in \{0,1\} \) indicates whether a purchase
occurred, whereas in personalized medicine, \( Y \) may represent a continuous
health metric such as blood sugar level \citep{johansson2016learning}.

We observe \( n \) i.i.d.\ samples \( \{(X_i, T_i, Y_i)\}_{i=1}^n \), where
\( X_i \) are covariates, \( T_i \in [m] \) is the action assigned by the
historical policy, and
\( Y_i := Y_i(T_i) \) is the observed outcome. The data are
\emph{observational}: for each subject, only the factual outcome
\( Y_i(T_i) \) is observed, while counterfactual outcomes
\( \{Y_i(t) : t \neq T_i\} \) remain unobserved. This defines the central
challenge of causal decision-making, i.e., learning optimal actions when only one
potential outcome per subject is observed.

To identify causal effects from observational data, we impose the standard
assumptions of unconfoundedness, positivity, and consistency. Specifically,
\emph{(unconfoundedness)} potential outcomes are conditionally independent of
treatment assignment given covariates, \( Y(t) \perp\!\!\!\perp T \mid X \);
\emph{(positivity)} each action occurs with positive probability over the
relevant covariate space, \( \mathbb{P}(T=t \mid X=x) > 0 \); and
\emph{(consistency)} observed outcomes coincide with the corresponding potential
outcomes under the realized treatment (i.e., if \( T_i=t \), then
\( Y_i=Y_i(t) \)).

These assumptions are commonly invoked in operational settings where actions are
determined by explicit policies based on observed contextual features.
In pricing, marketing, and recommendation systems, treatment assignment is often
driven by recorded covariates (e.g., booking time, fare class, or demand
signals), supporting unconfoundedness. Positivity is typically ensured through
initial experimentation or policy variation that prevents actions from being
deterministically excluded.

\subsection{The COAT Framework: Estimation and Prescriptive Optimization}

Our objective is to learn a policy
\( \tau: \mathbb{R}^k \rightarrow [m] \) that maps covariates \( X \) to an action
\( \tau(X) \) in order to maximize the expected counterfactual outcome. For
interpretability and operational transparency, we restrict \( \tau \) to a
feasible class of tree-based policies \( \mathcal{T} \), where each internal node
splits on a covariate and each leaf prescribes an action. We refer to such
policies as \emph{action trees} or \emph{prescriptive trees}.

The learning process proceeds in two stages:
\begin{enumerate}
    \item \textbf{Counterfactual Estimation.} Estimate the potential outcome model
    \( f(X,t) \) using observational data.
    \item \textbf{Action Tree Optimization.} Optimize over a constrained class of
    feasible trees \( \mathcal{T}_{\mathrm{feas}} \subseteq \mathcal{T} \) to
    identify a policy that maximizes the expected objective.
\end{enumerate}

Formally, the policy optimization problem is
\begin{equation}
\label{policy_optimization_tree}
\tau^*
=
\argmax_{\tau \in \mathcal{T}_{\mathrm{feas}}}
\;
\mathbb{E}_{X}\!\left[g\!\left(f(X, \tau(X))\right)\right],
\end{equation}
where \( g(\cdot) \) denotes a possibly transformed objective of interest. For
example, in pricing, \( f(X,t) \) represents purchase probability and
\( g(f(X,t)) = t \cdot f(X,t) \) captures expected revenue. In practice, the
expectation in~\eqref{policy_optimization_tree} is approximated by the empirical
average over the observed samples.

This modular design facilitates real-world deployment: counterfactual estimation
and policy optimization are decoupled, allowing flexible predictive models
upstream and a model-agnostic optimization downstream. Moreover, the feature sets
used in the two stages may often differ, i.e., rich features can be used for estimation, while
the action tree can depend on a smaller, deployment-ready subset suitable for
real-time implementation.

While accurate counterfactual estimation is essential, our primary contribution
lies in the optimization of interpretable action trees under operational
constraints; accordingly, we treat the estimation step as modular and focus the
remainder of the paper on the policy learning problem.

\subsection{Counterfactual Outcome Estimation}
\label{sect:counterfactual_estimation}

We aim to estimate the counterfactual outcome model
\[
f(X,t) := \mathbb{E}[Y(t) \mid X],
\]
where \( Y(t) \) denotes the potential outcome under action \( t \). Under the
assumptions of consistency, unconfoundedness, and positivity, this quantity is identified from
observational data as
\[
\mathbb{E}[Y(t) \mid X] = \mathbb{E}[Y \mid X, T=t].
\]
In the observational setting, however, only the factual outcome
\( Y_i = Y_i(T_i) \) is observed for each sample, and non-random treatment
assignment may induce selection bias, rendering naïve supervised learning
inconsistent.

A variety of causal estimation methods have been proposed to address this
challenge. Matching methods reduce confounding by pairing treated and untreated
samples with similar covariates, but they become inefficient in high-dimensional
settings \citep{stuart2010matching,abadie2006large}. Inverse propensity weighting
(IPW) reweights observations by estimated treatment probabilities to achieve
covariate balance, though extreme propensities can substantially inflate variance
\citep{rosenbaum1987model,hirano2003efficient}. Outcome modeling approaches,
including S-, T-, and X-learners, directly estimate potential outcomes using
machine learning models \citep{kunzel2019metalearners,nie2021quasi}.

In this work, we adopt a doubly robust (DR) estimator \citep{dudik2011doubly},
which combines outcome modeling with propensity score reweighting and yields
consistent estimates if either component is correctly specified. Doubly robust
estimators have also been successfully incorporated into prescriptive tree learning frameworks
\citep{biggs2020model,amram2020optimal,jo2021learning}.

\subsection{Action Tree Optimization Problem}

\begin{figure}
 \centering
  \includegraphics[width=.6\linewidth]
  {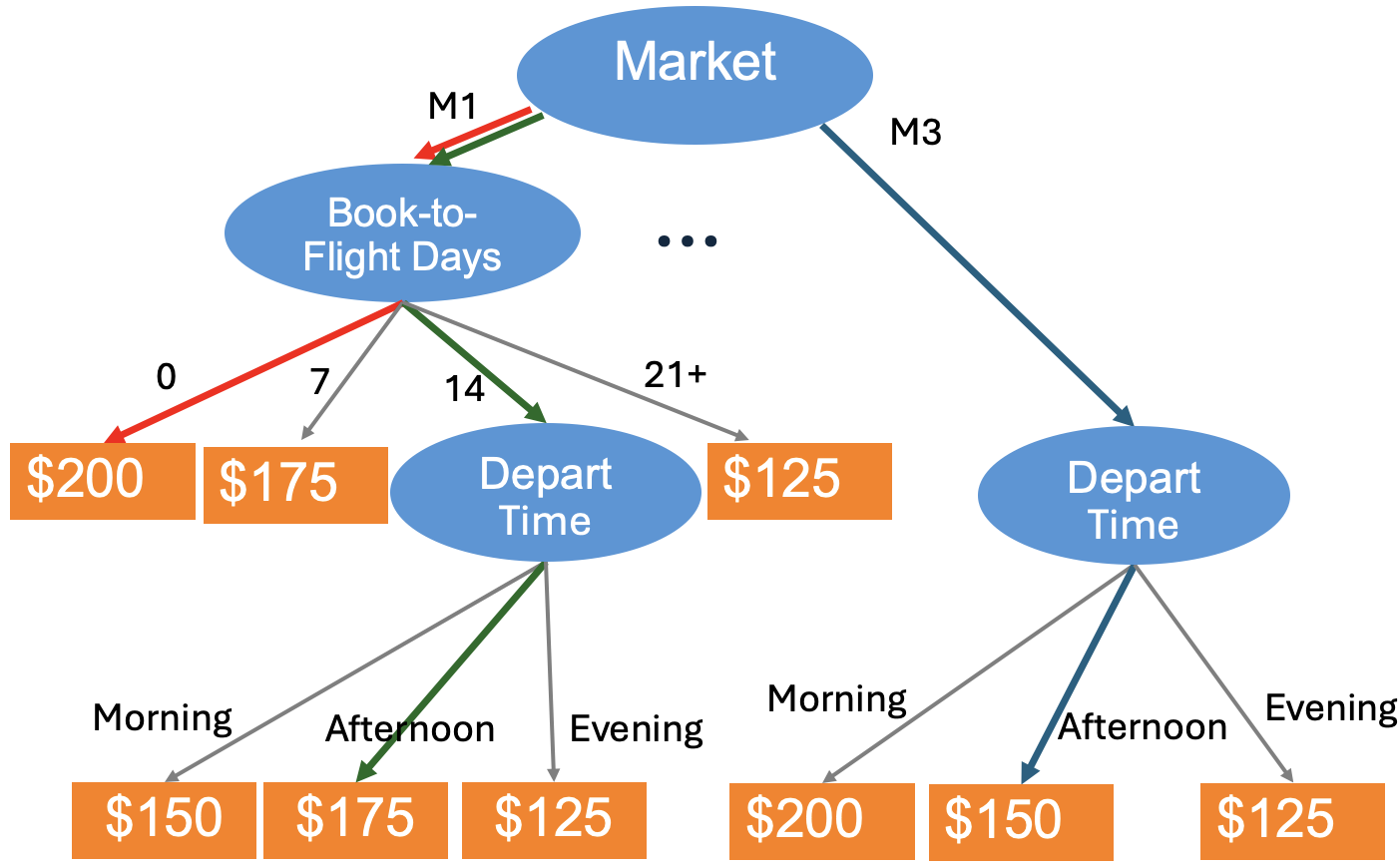}
  \caption{An example of a multiway-split prescriptive tree for airline ancillary pricing}
  \label{multiway_split}
\end{figure}

Given the estimated counterfactual outcomes \( \hat f(X_i,t) \), we compute
\( g_{i,t} := g(\hat f(X_i,t)) \) for each sample \( i \) and action \( t \). This
defines the unconstrained, individual-level optimal action:
\begin{equation}
\label{individual_policy}
\tau^*(X_i) = \argmax_{t \in [m]} g_{i,t}.
\end{equation}
While this policy maximizes outcomes on a per-sample basis, it is complex and
impractical to deploy.

We therefore constrain \( \tau \) to interpretable decision trees. Prior work has
focused on binary trees of fixed depth \( d \), denoted \( \mathcal{T}_B(d) \)
\citep{biggs2020model,amram2020optimal,jo2021learning}. We generalize this to multiway-split trees
with at most \( l \) leaves, denoted \( \mathcal{T}_M(l) \) (see Figure~\ref{multiway_split}). A multiway tree with
\( l \) leaves has expressiveness comparable to a binary tree of depth
\( d \) satisfying \( 2^d = l \), while often yielding more compact
representations.

The resulting empirical optimization problem is
\begin{equation}
\label{tree_optimization}
\max_{\tau \in \mathcal{T}_M(l)} \sum_{i=1}^n g_{i,\tau(X_i)}.
\end{equation}

In the next section, we present a mixed-integer programming formulation
and an efficient column-generation algorithm that exploit the structure
of multiway-split trees by representing candidate policies as paths in a
policy graph.

\section{Policy Optimization via Path-Based Trees}
To learn an interpretable policy, we construct a space of candidate decision
rules represented as paths in a policy graph
(Figure~\ref{feature_space}). Each path corresponds to a rule that assigns an
action based on a subset of features. The policy learning task then reduces to
selecting a small, non-overlapping subset of paths that together form a valid
decision tree and maximize the overall counterfactual objective. This selection
problem is formulated as a mixed-integer program (MIP), whose solution induces
an interpretable prescriptive tree (Figure~\ref{multiway_split}).

In this section, we first describe the construction of the feature path space and the
resulting optimization formulation. The solution algorithm, based on column
generation, is presented in Section~\ref{sect_algorithm}. We then show, in
Section~\ref{sect_constraint}, how this framework can be extended to incorporate
business and operational constraints.

\subsection{Policy Graph Construction}\label{sect_rule_space}

We construct an acyclic multi level directed graph \( G \) where each node represents a unique feature value. Each feature corresponds to a distinct level in the graph, with its values represented as nodes, and nodes from one level fully connected to those in the next. The graph also includes a source and a sink node. For every feature except the action feature \( \pi \), we introduce a dummy node labeled \emph{SKIP}. Figure~\ref{feature_space} illustrates an example with three features: \textit{Market}, \textit{Book to Flight Days}, which indicates how far in advance the ticket is purchased, \textit{Depart Time}, and \textit{Price} as the action feature.

A policy (or decision rule) corresponds to a path \( p \in \mathcal{P} \) from the source to the sink in \( G \). When a path traverses a \emph{SKIP} node, the associated feature is omitted from that rule. The inclusion of \emph{SKIP} nodes allows paths to represent all possible feature combinations, thereby defining a comprehensive rule set \( \mathcal{P} \) from which the optimal subset of policies will be selected.

\begin{figure}
 \centering
  \includegraphics[width=.7\linewidth]
  {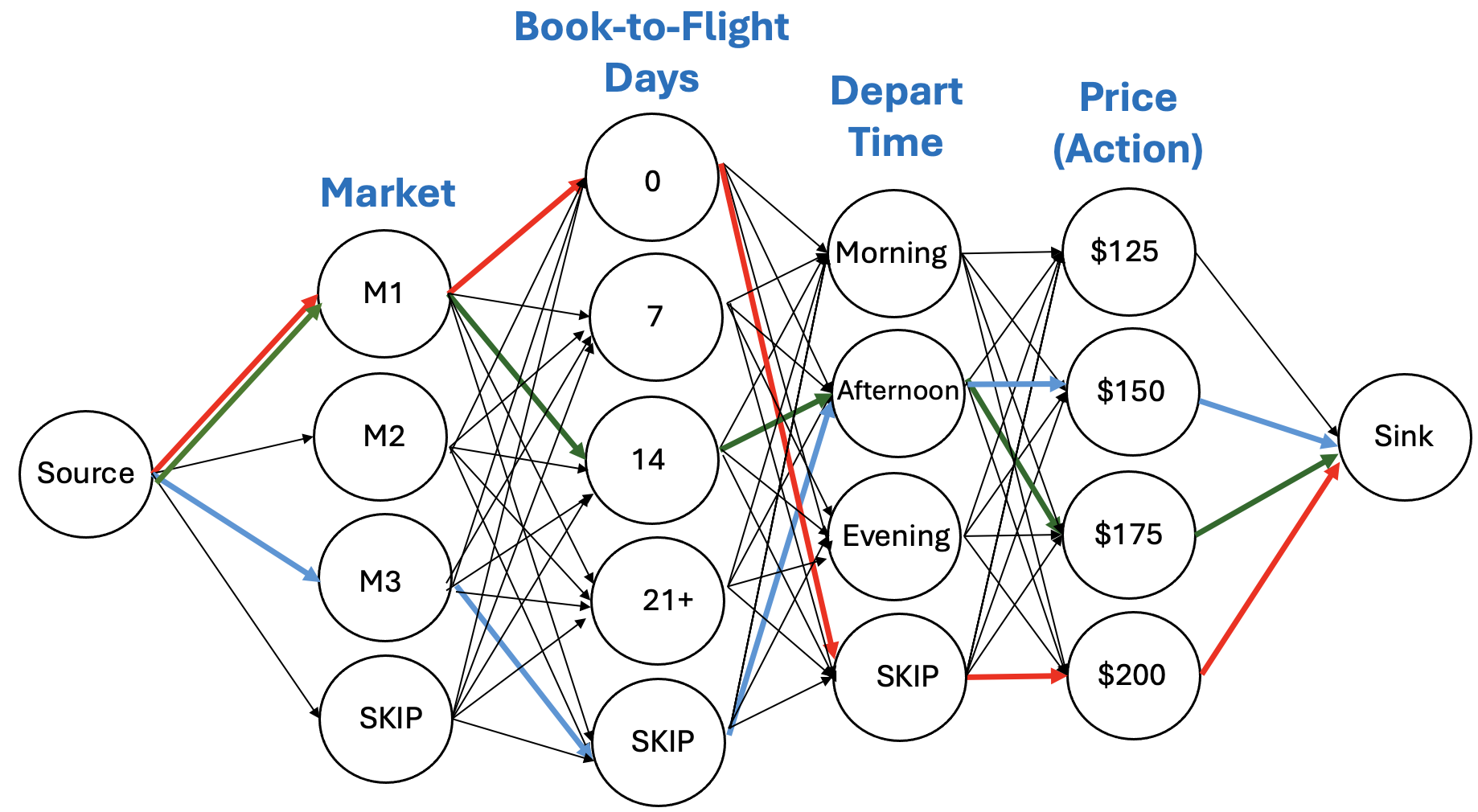}
  \caption{Policy graph with three features and \textsc{Price} as the action }
  \label{feature_space}
\end{figure}

Given any path \( p \), we can determine the set of samples that satisfy its conditions and compute their associated counterfactual outcomes. The objective is to select a subset of rules that maximizes the total counterfactual outcome. However, because the policy space grows exponentially with the number of features \( k \), naive enumeration becomes infeasible for high dimensional data. Section~\ref{sect_algorithm} presents an algorithm that efficiently navigates this space by dynamically generating paths as needed.

For clarity, we assume throughout this section that all features are categorical. In Section~\ref{subsect_numerical_features}, we extend the framework to handle numerical features through discretization and adaptive splitting.

\subsection{Mixed-Integer Programming Formulation}
\label{sect_MIP}

By construction, a sample may satisfy multiple candidate paths in the policy
graph, for example, a path that matches its exact feature values as well as
other paths that traverse \textsc{SKIP} nodes. To recover a valid tree,
however, the selected paths must form a non-overlapping partition of the sample
space, so that each sample is ultimately assigned to exactly one rule
\citep{Furnkranz2017}. We enforce this structure through a mixed-integer
programming (MIP) formulation whose feasible solutions correspond to valid
multiway decision trees.

\paragraph{Candidate policies.}
Let \( j = 1, \ldots, |\mathcal{P}| \) index the set of candidate policies (paths)
derived from the policy graph. Each policy \( j \) prescribes an action
\( t_j \in [m] \) and covers a subset of samples
\( S_j \subseteq \{1,\dots,n\} \), where a sample \( i \in S_j \) satisfies all
feature conditions along path \( j \).

The total reward associated with policy \( j \) is defined as
\[
r_j = \sum_{i \in S_j} g_{i,t_j},
\]
where \( g_{i,t_j} \) denotes the predicted reward of assigning action \( t_j \)
to sample \( i \).

Define the binary incidence matrix \( a_{ij} \) by
\[
a_{ij} =
\begin{cases}
1 & \text{if sample } i \text{ satisfies policy } j, \\
0 & \text{otherwise}.
\end{cases}
\]
Then the policy reward can be written compactly as
\[
r_j = \sum_{i=1}^{n} a_{ij} g_{i,t_j}.
\]

\paragraph{Decision variables.}
We introduce the following variables:
\begin{itemize}
    \item \( z_j \in \{0,1\} \), indicating whether policy \( j \) is selected;
    \item \( s_i \ge 0 \), a slack variable associated with sample \( i \).
\end{itemize}

Slack variables are introduced to ensure feasibility of the restricted master
problem during column generation when coverage may be temporarily incomplete.
In practice, we initialize the algorithm with a trivial path that assigns a common
action to all samples, so that all slack variables are zero.
At optimality, the penalty coefficients \( c_i \) are chosen sufficiently large
to ensure that all slack variables are driven to zero, thereby enforcing exact
coverage.

\paragraph{The \textsc{COAT} formulation.}
The Counterfactual Optimal Action Tree problem is formulated as:
\begin{align}
\textbf{(COAT)} \quad
\max \; & \sum_{j=1}^{|\mathcal{P}|} r_j z_j - \sum_{i=1}^{n} c_i s_i
\label{objective} \\
\text{s.t.} \quad
    & \sum_{j=1}^{|\mathcal{P}|} a_{ij} z_j + s_i = 1,
        && \forall i = 1,\dots,n
        && \text{(Coverage)} \label{constraint_coverage} \\
    & \sum_{j=1}^{|\mathcal{P}|} z_j \le l,
        && && \text{(Tree size)} \label{constraint_capacity} \\
    & z \in \mathcal{Z},
        && && \text{(Additional constraints)} \label{constraint_additional} \\
    & z_j \in \{0,1\},
        && \forall j = 1,\dots,|\mathcal{P}| \nonumber \\
    & s_i \ge 0,
        && \forall i = 1,\dots,n. \nonumber
\end{align}

Here, \( l \) is a user-defined upper bound on the number of selected policies
(i.e., leaf nodes), and \( \mathcal{Z} \) represents additional structural or
business constraints 
which we detail in Section~\ref{sect_constraint}.

While a training sample may satisfy multiple candidate policies due to
\textsc{SKIP} nodes, constraint~\eqref{constraint_coverage} ensures that, at
optimality, each sample is assigned to exactly one selected policy. At test
time, if a previously unseen feature combination arises, we assign a default
action or match the observation to the closest training samples to determine an
appropriate policy.

\begin{theorem}[Equivalence to Multiway Decision Trees]
\label{them:equivalence}
Under the \textsc{COAT} formulation, any optimal solution \( \mathbf{z}^* \)
induces a collection of active paths
\( \mathcal{P}(\mathbf{z}^*) \subseteq \mathcal{P} \) that can be represented as a
feasible multiway decision tree \( \mathcal{T}_M^*(l) \) with at most \( l \)
leaves. Moreover, the objective value attained by \( \mathcal{T}_M^*(l) \)
coincides with the optimal objective value of the \textsc{COAT} problem.
\end{theorem}

For clarity of exposition, proofs of all theoretical results are deferred to the
electronic companion.

While the \textsc{COAT} formulation admits a compact conceptual representation,
the number of candidate paths grows exponentially with the number of features,
rendering direct solution of the MIP 
intractable in large-scale settings. In the next section, we develop an efficient
solution approach that dynamically constructs only the most promising
policies.

\subsection{Scalable Algorithm via Column Generation}
\label{sect_algorithm}

The \textsc{COAT} formulation is computationally challenging due to the
combinatorial nature of jointly selecting a set of non-overlapping decision
rules and their associated actions, which subsumes the classical set
partitioning problem. Because the number of candidate paths grows exponentially
with the number of features, explicitly enumerating all policies is infeasible
in large-scale settings.

To address this challenge, we develop a column generation (CG) algorithm that
exploits the path-based structure of the formulation. Rather than enumerating
all candidate rules, CG iteratively solves a restricted master problem over a
small subset of paths and dynamically generates new, high-value paths through a
pricing subproblem.

This approach is enabled by the path-based representation underlying
\textsc{COAT}. In contrast to arc-based decision tree formulations, which require
binary variables for each possible branch at each depth and explicitly model
sample-to-node assignments, resulting in formulations with
\( \mathcal{O}(2^d n) \) variables, where \( d \) denotes the depth of the
binary tree. Our formulation operates directly at the level
of complete root-to-leaf paths. This structure admits a natural decomposition
between rule selection and rule generation, making CG an effective and scalable
solution method.

Comprehensive computational experiments establishing the scalability and solution-quality advantages of this path-based formulation over arc-based and other prescriptive-tree methods, on both synthetic and real datasets, were reported in the conference version of this work~\citep{AnonymousAAAI2022}. The present paper takes that formulation as given and focuses on theoretical extensions, deployment-oriented design, and field validation.

Column generation begins with a
restricted master problem (RMP) containing only a small subset of variables and
iteratively augments it by identifying improving variables through a pricing
subproblem. The process terminates when no variables with positive reduced cost
remain, at which point the RMP solution is optimal for the linear relaxation of
the full problem.

\subsubsection*{Restricted Master Problem (RMP).}

Let \( \mathcal{P}' \subseteq \mathcal{P} \) denote the current subset of candidate
paths included in the RMP. Relaxing integrality, we allow
\( z_j \in [0,1] \) for all \( j \in \mathcal{P}' \). The RMP is formulated as:
\begin{align}
\textbf{(RMP)} \quad \max \quad
    & \sum_{j \in \mathcal{P}'} r_j z_j - \sum_{i=1}^n c_i s_i \nonumber \\
\text{s.t.} \quad
    & \sum_{j \in \mathcal{P}'} a_{ij} z_j + s_i = 1,
        && \forall i = 1, \dots, n,
        \label{constraint_coverage_rmp} \\
    & \sum_{j \in \mathcal{P}'} z_j \le l,
        \label{constraint_capacity_rmp} \\
    & z_j \ge 0,
        && \forall j \in \mathcal{P}', \nonumber \\
    & s_i \ge 0,
        && \forall i = 1, \dots, n. \nonumber
\end{align}


The additional constraint set \( \mathcal{Z} \) introduced in
Eq.~\eqref{constraint_additional} is temporarily omitted here and reintroduced
in Section~\ref{sect_constraint}.

\subsubsection*{Pricing Subproblem.}

Let \( \lambda_i \) and \( \mu \) denote the dual variables associated with the
coverage constraints~\eqref{constraint_coverage_rmp} and the cardinality
constraint~\eqref{constraint_capacity_rmp}, respectively. The dual of the RMP is:
\begin{align}
\textbf{(Dual of RMP)} \quad \min \quad
    & \sum_{i=1}^n \lambda_i + l \mu \nonumber \\
\text{s.t.} \quad
    & \sum_{i=1}^n a_{ij} \lambda_i + \mu \ge r_j,
        && \forall j \in \mathcal{P}', \label{dual_constraint} \\
    & \lambda_i \ge -c_i,
        && \forall i = 1, \dots, n, \nonumber \\
    & \mu \ge 0. \nonumber
\end{align}

At optimality of the RMP, all paths \( j \in \mathcal{P}' \) satisfy dual
feasibility. Any path \( j \in \mathcal{P} \setminus \mathcal{P}' \) that violates
constraint~\eqref{dual_constraint} has positive reduced cost
\begin{equation}
\pi_j
= r_j - \left( \sum_{i=1}^n a_{ij} \lambda_i + \mu \right),
\label{eq_reduced_cost}
\end{equation}
and can improve the RMP objective when added.

The pricing subproblem therefore seeks paths with \( \pi_j > 0 \). This can be
formulated as 
a shortest-path problem with appropriately defined arc costs,
where the total path cost corresponds to \( -\pi_j \)
\citep{horne1980finding,irnich2005shortest}. In practice, it suffices to identify
a small number of paths with positive reduced cost at each iteration. 

\subsubsection*{CG Iteration and Final MIP.}

The column generation procedure alternates between solving the RMP and the
pricing subproblem until no paths with positive reduced cost are identified or a
predefined iteration limit is reached. At termination, the RMP solution is
optimal for the linear relaxation of the full \textsc{COAT} problem. Dual stabilization and regularization terms are added to mitigate stalling and accelerate convergence for large-scale instances~\citep{subramanian2008effective}.

Finally, integrality constraints on \( \mathbf{z} \) are reinstated, and the MIP
is solved over the generated column set using branch-and-bound or
branch-and-price techniques~\citep{barnhart1998branch} to obtain an integral
solution. In the resulting optimal solution, all slack variables are zero and
the selected paths define a valid multiway decision tree.

\subsection{Incorporating Operational Constraints}
\label{sect_constraint}

A key strength of the \textsc{COAT} framework lies in the flexibility of its
path-based formulation, which enables the incorporation of a wide range of
constraints introduced abstractly in
Eq.~\eqref{constraint_additional}, ensure that the learned policies are not only
optimal with respect to counterfactual outcomes, but also feasible,
interpretable, and aligned with real-world operational requirements.
We categorize  constraints into two types based on their scope:
\emph{inter-rule constraints}, 
and \emph{intra-rule constraints}, which interact with the column generation
algorithm in fundamentally different ways.

\subsubsection{Inter-Rule Constraints}

Inter-rule constraints govern relationships across multiple rules and
are enforced at the master problem level. A basic example is the cardinality
constraint in Eq.~\eqref{constraint_capacity}, which limits the total number of
selected rules.

More generally, inter-rule constraints from
Eq.~\eqref{constraint_additional} can be expressed as linear inequalities over
the rule selection variables \( \mathbf{z} \):
\begin{equation}
    \underline{q}_m \leq \sum_{j=1}^{|\mathcal{P}|} \rho_{mj} z_j \leq
    \overline{q}_m, \quad \forall m = 1, \dots, M,
\end{equation}
where \( \rho_{mj} \) captures the contribution of rule \( j \) to constraint
\( m \), and \( \underline{q}_m \) and \( \overline{q}_m \) represent lower and
upper bounds, respectively.

Because each selected path corresponds to a mutually exclusive leaf of the
decision tree, inter-rule constraints can naturally be applied to arbitrary
\emph{subsets of rules}, rather than uniformly across the entire population.
This is particularly important in practice, where operational limits often apply
only to specific segments, markets, or actions.
In the airline pricing application (Section~\ref{sect_case_study}), for example,
\( \rho_{mj} \) may represent the expected demand contributed by rule \( j \) in
market \( m \), while \( \underline{q}_m \) and \( \overline{q}_m \) encode minimum
sales targets, capacity limits, or business commitments that apply only to
certain routes or booking contexts.

To incorporate inter-rule constraints into the column generation algorithm, let
\( \overline{\tau}_m \geq 0 \) and \( \underline{\tau}_m \geq 0 \) denote the dual
variables associated with the upper and lower bounds, respectively. These dual
variables modify the reduced cost of a candidate path \( j \) as follows:
\begin{equation}
\pi_j
= r_j - \left(
\sum_{i=1}^n a_{ij} \lambda_i
+ \mu
+ \sum_{m=1}^M \rho_{mj} (\overline{\tau}_m - \underline{\tau}_m)
\right).
\end{equation}
As a result, inter-rule constraints directly influence the pricing subproblem by
guiding the generation of paths toward globally feasible and operationally
compliant solutions.

\subsubsection{Intra-Rule Constraints}

Intra-rule constraints restrict the internal structure of individual decision
rules to ensure interpretability and consistency with business logic. Unlike
inter-rule constraints, they do not couple multiple rules and are therefore
enforced locally within the pricing subproblem.

Typical examples include limits on rule length (analogous to tree depth),
minimum sample coverage requirements, and restrictions on allowable combinations
of features or actions. Some of these constraints are difficult, or impossible, to
express as linear constraints at the arc level. For instance, business rules such
as ``features A and B must not appear together,'' or ``features B and C must
either both appear or both be absent,'' impose logical dependencies across
multiple features within a single rule. Encoding such constraints in
arc-based formulations would require nontrivial linearizations and additional
auxiliary variables.

In the \textsc{COAT} framework, intra-rule constraints are naturally enforced
during path construction. As candidate paths are incrementally extended on the
policy graph \( G \), feasibility checks are applied at each step, and
infeasible paths are immediately pruned. 
By enforcing intra-rule constraints directly within the pricing subproblem, the
algorithm avoids explicit linearization of complex logical conditions, reduces
the effective search space, and preserves the interpretability of the resulting
decision rules.

\medskip
This separation between inter-rule and intra-rule constraints, coupled with their
integration at different stages of the CG procedure, highlights a
key advantage of the path-based formulation: it supports rich, context-dependent
operational constraints while preserving scalability and algorithmic
tractability.

\remark{Such flexibility is difficult to achieve in arc-based tree formulations. Because
arc-based models do not explicitly represent complete root-to-leaf paths,
constraints that apply to entire decision rules or to subsets of the population
defined by feature combinations are cumbersome to express and often require
auxiliary variables and complex linearizations. In contrast, \textsc{COAT}
represents each path as a complete, semantically meaningful rule, allowing
constraints to be imposed naturally both across rules and within individual
rules.}

\section{Design for Scalable and Deployable Prescriptive Policies}
\label{sect_comp_tractability}

Large-scale deployment of prescriptive decision policies imposes requirements that
extend beyond computational efficiency, including interpretability, governance,
and real-time execution under operational constraints.
This section examines how the \textsc{COAT} framework is explicitly designed to
respect such deployment constraints while remaining computationally tractable at
scale. The results in this section are deliberately modest in their mathematical
machinery; their role is to make precise why specific design choices, namely
bounded depth, prescriptive-feature restriction, and equality-based categorical
tests, yield tractable, auditable policies in practice, rather than to establish
new optimization theory.

\subsection{Depth Constraints as an Operational Design Choice}
\label{subsec:depth_constraints}

In real-world decision systems, policy complexity is constrained not only by
computational considerations, but also by interpretability, latency, and
governance requirements. Policies should be sufficiently simple to be audited, explained to stakeholders, and
executed reliably in real time. These considerations naturally limit the number
of features that a policy may condition on. In tree-based
representations, this translates directly into bounding the depth of the
tree. 

\paragraph{Implications for the Policy Search Space.}
To formalize the effect of depth constraints, consider a setting with \(k\)
prescriptive features, each taking at most \(\eta\) distinct values. Each path in
the decision tree corresponds to a rule defined by selecting a subset of
features and assigning values to them.

\begin{proposition}[Size of the Unconstrained Path Space]
\label{prop_feature_space}
The number of possible paths in the unconstrained search space satisfies
$|\mathcal{P}| = \mathcal{O}(\eta^k)$.
\end{proposition}

Proposition~\ref{prop_feature_space} highlights the fundamental combinatorial
challenge faced by prescriptive policy optimization: the number of candidate
rules grows exponentially with the number of features, rendering exhaustive
enumeration infeasible even for moderate values of \(k\).

\paragraph{Depth-Constrained Policies.}
A depth constraint limits the number of features that a policy rule may
reference. When the depth of the decision tree is restricted to \(d\), each
feasible path may involve at most \(d\) feature--value assignments. Let
\(\mathcal{P}_d \subseteq \mathcal{P}\) denote the set of paths with at most \(d\)
active feature conditions.

\begin{proposition}[Size of the Depth-Constrained Path Space]
\label{prop_constrained_space}
If the depth of the decision tree is restricted to \(d\), the number of feasible
paths satisfies $
|\mathcal{P}_d|
= \mathcal{O}\!\left(\binom{k}{d}\eta^d\right)$.
\end{proposition}

Proposition~\ref{prop_constrained_space} shows that depth constraints
simultaneously enforce operational simplicity and dramatically reduce the size
of the policy search space.  For example, with binary features ($\eta=2$) and depth $d=2$, the unconstrained
space grows as $\mathcal{O}(2^k)$, whereas the depth-constrained space grows as
$\mathcal{O}(k^2)$.

\paragraph{Optimization Implications of Depth Constraints.}
Depth constraints also yield favorable computational properties: for fixed \(d\)
and bounded \(\eta\), the LP relaxation of the depth-constrained \textsc{COAT}
model has polynomial size and can be solved in polynomial time.

\begin{proposition}[Polynomial-Time LP Relaxation for Fixed Depth]
\label{prop_lp_upper_bound}
For fixed depth \(d\) and bounded feature cardinality \(\eta\), the LP relaxation
of the depth-constrained \textsc{COAT} formulation can be solved in time
polynomial in \(k\) and \(n\).
\end{proposition}

Depth constraints therefore serve a dual role: they encode operational simplicity
while placing policy optimization within a computationally tractable regime.

\subsection{Exact State Aggregation in Policy Prescription}
\label{sec:state_aggregation}

A key advantage of the \textsc{COAT} framework is its decoupled design: the
counterfactual outcome estimator may depend on the full feature vector
\(X \in \mathbb{R}^k\), while the prescriptive policy is restricted to a smaller,
operationally feasible subset of features. This separation reflects practical
deployment constraints, as many predictive features may be unavailable or 
inadmissible at decision time.

In the airline pricing application (Section~\ref{sect_case_study}), purchase
probabilities were estimated using a rich set of booking, flight, and passenger
attributes, while the prescriptive pricing policy was restricted to a small,
interpretable feature set consistent with regulatory fare filing rules. In other
applications, prescriptive features must also be available in real time, as
recommendations need to be computed and displayed instantaneously when a customer
searches for a product. Restricting the policy to real-time–accessible features
enabled low-latency deployment while still leveraging rich upstream predictive
models.

This asymmetry implies that many observations share identical prescriptive
features but differ only in predictive attributes. Such observations are
\emph{decision-equivalent}: they satisfy the same set of feasible policy rules and
differ only in their contribution to the objective. They can therefore be
aggregated without loss of feasibility or optimality, allowing the policy
optimization problem to be reformulated exactly over unique prescriptive feature
profiles and yielding a strictly smaller but equivalent problem.

Decision-equivalence does not imply outcome-equivalence. Although observations
within the same group satisfy the same prescriptive rules, their counterfactual
rewards may differ because the estimator \(f(X,t)\) is evaluated on the full
feature vector \(X\), not on the reduced prescriptive representation \(Z\).
Aggregation thus preserves heterogeneity in expected outcomes while restricting
only the structure of the decision policy. This distinction is central to
\textsc{COAT}: predictive models can exploit rich feature sets without constraining
the interpretability or deployability of the learned policy.

\paragraph{Grouping by Prescriptive Feature Profiles.}
Let \( Z_i := R(X_i) \in \Omega \) denote the prescriptive feature vector, where
\( R(\cdot) \) restricts \( X_i \) to features available at decision time and
\( k' \ll k \). Define the equivalence relation \( i \sim i' \) if and only if
\( Z_i = Z_{i'} \), which partitions the dataset into groups
\( \mathcal{G}_1,\ldots,\mathcal{G}_{N'} \), where
\( N' = \lvert \{ Z_i : i=1,\ldots,n \} \rvert \).

Because candidate policy rules depend only on \( Z \), all observations within a
group \( \mathcal{G}_{i'} \) satisfy the same set of policy paths, implying that
the coverage indicator is constant within each group. We therefore define
\( a_{i'j} := a_{ij} \) for any \( i \in \mathcal{G}_{i'} \).

For each group \( i' \), define the aggregated penalty
\( c_{i'} := \sum_{i \in \mathcal{G}_{i'}} c_i \) and aggregated reward under policy
\( j \) as \( r_{i'j} := \sum_{i \in \mathcal{G}_{i'}} g_{i,t_j} \). The total reward
of policy \( j \) can then be written as
\( r_j = \sum_{i'=1}^{N'} a_{i'j} r_{i'j} = \sum_{i=1}^{n} a_{ij} g_{i,t_j} \),
which coincides exactly with the original definition. Aggregation therefore
preserves both feasibility and objective value for every feasible policy. The
full reduced mixed-integer formulation and proof of equivalence are provided in
Section~\ref{ec:state_aggregation}.

\paragraph{Computational Implications}

All computational gains arise from an exact reduction in problem dimension rather
than approximation. In the original formulation, the coverage constraints take
the form \( A z + s = \mathbf{1} \), with
\( A \in \{0,1\}^{n \times |\mathcal{P}|} \). After aggregation, the reduced
formulation becomes \( A' z + s' = \mathbf{1} \), where
\( A' \in \{0,1\}^{N' \times |\mathcal{P}|} \), with the same decision variables
\( z \) but substantially fewer constraints since \( N' \ll n \).

This reduction improves the conditioning of LP relaxations and accelerates column
generation: the dual vector is shorter, reduced-cost computations are cheaper,
and pricing problems become smaller. The final branch-and-bound phase also
benefits from fewer constraints and tighter relaxations.

\begin{theorem}[Reduction in Problem Size from Prescriptive Dimensionality]
\label{thm:reduction_speedup}
Let the dataset contain \(n\) samples and \(k\) categorical features, each taking
at most \(\eta\) distinct values. Suppose the prescriptive feature set is
restricted to \(k' < k\) features. Then:
\begin{enumerate}
\item The number of unique prescriptive feature profiles satisfies
$
N' \le \eta^{k'}$.
\item Consequently, the reduced \textsc{COAT} formulation contains at most
\(\eta^{k'}\) coverage constraints. In particular, when \(n \ge \eta^{k'}\), the
number of coverage constraints is independent of \(n\), and the problem size is
reduced by at least
$\frac{n}{N'} \ge \frac{n}{\eta^{k'}}$.
\end{enumerate}
\end{theorem}

Theorem~\ref{thm:reduction_speedup} shows that the number of coverage constraints
in the prescriptive optimization problem scales with the \emph{number of distinct
prescriptive states}, rather than with the total number of samples. Once the
dataset is sufficiently large to populate all feasible prescriptive feature
combinations, the number of coverage constraints---and hence the row dimension of
the master problem---becomes independent of \(n\). The number of candidate columns
\(|\mathcal{P}|\) and the size of the pricing subproblem are instead governed by the
policy-graph dimensions and are unaffected by this reduction.

\remark{
Aggregation is not unique to \textsc{COAT}: arc-based tree formulations may also
collapse observations with identical feature vectors. However, such aggregation
typically occurs after the tree structure is fixed and does not reduce the
underlying structural complexity of the model. In contrast, aggregation in
\textsc{COAT} operates directly at the level of prescriptive states prior to
policy construction, yielding an exact reduction in both data dimension and
effective search space.
}

\subsection{Data Degeneracy and Reachable Prescriptive States}
\label{subsec:data_degeneracy}

In operational settings, the realized data typically occupies only a small and
highly structured subset of the combinatorial prescriptive feature space. Let
\( Z \in \Omega \) denote the prescriptive feature vector, and define the
empirical support induced by the observed data as
 $\widehat{\Omega} := \{ Z_i : i = 1,\dots,n \} \subseteq \Omega$.

In practice, \( |\widehat{\Omega}| \ll |\Omega| \), reflecting dependencies among
features, business rules, instrumentation constraints, and regulatory
requirements. We refer to this phenomenon as \emph{data degeneracy}.

In the \textsc{COAT} framework, prescriptive policies are represented as paths in
a policy graph, where each path corresponds to a (possibly partial) assignment
of prescriptive features. Let \( \mathcal{S}_\ell \) denote the set of all
prescriptive states at depth \( \ell \), i.e., partial assignments involving
\( \ell \) feature--value conditions. We define the set of \emph{reachable}
prescriptive states as
$\widehat{\mathcal{S}}_\ell
:= \bigl\{ s \in \mathcal{S}_\ell : \exists\, i \text{ such that } Z_i
\text{ satisfies } s \bigr\}$.
States outside \( \widehat{\mathcal{S}}_\ell \) have zero empirical support and
cannot be satisfied by any observed sample.

Consequently, the policy search space can be restricted to paths that are
empirically reachable. Let \( \mathcal{P}_d \) denote the set of all depth-\(d\)
paths in the policy graph, and define the subset of reachable paths as
$\widehat{\mathcal{P}}_d
:= \bigl\{ p \in \mathcal{P}_d : \exists\, i \text{ such that } Z_i
\text{ satisfies } p \bigr\}$.
Paths containing unsupported partial assignments are excluded by construction.
As a result, the effective policy space satisfies $|\widehat{\mathcal{P}}_d| \ll |\mathcal{P}_d| \le |\mathcal{P}|$.

This reduction is further reinforced by deployment-driven modeling constraints.
Limits on tree depth restrict the complexity of individual rules, prescriptive
feature restrictions confine attention to deployable attributes, and
cardinality constraints bound the number of selected paths. Together, these
restrictions ensure that the optimization problem scales with the number of
\emph{reachable, decision-relevant prescriptive states}, rather than with the
size of the ambient feature space.

From a computational perspective, data degeneracy aligns naturally with column
generation. Pricing subproblems and reduced-cost computations are performed only
over \( \widehat{\mathcal{P}}_d \), and columns corresponding to unsupported paths
are never generated. This tight coupling between empirical support and policy
representation is a key reason why \textsc{COAT} remains tractable in
high-volume production environments, even when the upstream counterfactual
models operate over high-dimensional feature spaces.

\subsection{Representational Choices: Path-Based and Arc-Based Trees}
\label{sec:arc_vs_path}

The flexibility to incorporate operational constraints, the ability to perform
exact state aggregation, and the favorable interaction with data degeneracy are
not independent features of the \textsc{COAT} framework. Rather, they arise
directly from the underlying \emph{path-based} representation of prescriptive
policies. To clarify the role of this modeling choice, we contrast the
path-based formulation with traditional \emph{arc-based} decision tree
formulations and discuss the resulting design implications.


\paragraph{Equivalence in an Unrestricted Binary Setting.}
We first consider an unrestricted binary setting, in which all \(2^k\) feature
assignments can be represented.
 In this unrestricted setting, both arc-based and path-based
formulations can represent any mapping
\(\{0,1\}^k \rightarrow \mathcal{A}\), where \(\mathcal{A}\) denotes the action
set. Although the resulting tree structures and symbolic rules may differ, both
formulations can achieve the same optimal objective value.

\paragraph{Divergence under Structural Constraints.}
This equivalence is fragile and breaks down once structural constraints are
introduced. In an arc-based formulation, a depth constraint implicitly requires
simultaneous decisions over \emph{which} features appear along a path and the
\emph{order} in which those features are tested. As a result, depth constraints
induce a combinatorial search over both feature subsets and feature permutations.

In contrast, the path-based formulation underlying \textsc{COAT} imposes a fixed
canonical ordering of prescriptive features. A path may include or skip features,
but their order is predetermined. Depth constraints therefore limit only the
number of feature conditions appearing in a rule, without introducing additional
ordering decisions. This interpretation aligns rule depth directly with decision
complexity and operational simplicity, rather than with tree geometry.

\paragraph{Multiway Representation of Categorical Features.}
The treatment of categorical features further distinguishes arc-based and
path-based representations. In binary-split trees, a split on a categorical
feature requires selecting a subset of category values, introducing substantial
combinatorial complexity and often producing rules that are difficult to
interpret, audit, or justify operationally.

By contrast, \textsc{COAT} employs equality-based tests of the form
\(X_f = v\), yielding a fixed and transparent set of atomic conditions per
categorical feature. Multiway branching is realized implicitly through distinct
paths rather than explicit subset splits, eliminating unnecessary structural
ambiguity.

\begin{proposition}[Split Enumeration for Categorical Features]
\label{prop:multiway_split_complexity}
For a categorical feature with \(\eta\) distinct values, a binary-split tree
admits \(2^{\eta-1}-1\) distinct nontrivial splits, whereas the equality-based
multiway representation used in \textsc{COAT} admits exactly \(\eta\) atomic tests
of the form \(X_f = v\).
\end{proposition}

While Proposition~\ref{prop:multiway_split_complexity} highlights the combinatorial
explosion of categorical splits in binary trees, the next result formalizes how this combinatorial complexity translates into
additional depth in binary-split trees.

\begin{proposition}[Depth Overhead for Categorical Features in Binary trees]
\label{prop:depth_overhead_mip}
Consider $k$ categorical prescriptive features, each taking at most $\eta$ distinct
values. Let $\mathcal{T}_M$ be a multiway prescriptive tree in which each root--to--leaf
path tests at most $d$ features using equality conditions of the form $X_f = v$.
Then there exists a binary-split decision tree $\mathcal{T}_B$ that implements the
same policy with depth at most $d\lceil \log_2 \eta \rceil$. Conversely, there exist prescriptive policies over $\{1,\dots,\eta\}^k$ such that
any binary-split decision tree representation requires depth at least
$\lceil k \log_2 \eta \rceil$.
\end{proposition}

Proposition~\ref{prop:depth_overhead_mip} shows that depth in arc-based binary
tree formulations conflates decision complexity with categorical encoding,
whereas depth in the \textsc{COAT} framework directly reflects the number of
prescriptive conditions in a rule.

\paragraph{Implications for Scalable and Deployable Policy Optimization.}
Taken together, these distinctions show that arc-based and path-based
formulations coincide only in a narrow, unrestricted setting. Once depth limits,
categorical features, or operational constraints are introduced, the path-based formulation used in
\textsc{COAT} offers a more structured and interpretable representation. By fixing
feature order, avoiding subset-based categorical splits, and operating directly
on decision-relevant paths, the formulation complements the exact state
aggregation and data degeneracy properties discussed earlier. These design
choices help explain why \textsc{COAT} remains scalable and practical in
large-scale operational environments.

\section{Operational Deployment in Airline Ancillary Pricing}
\label{sect_case_study}

This section presents a real-world deployment of the \textsc{COAT} framework. We describe the operational constraints that shaped the modeling
choices, present the system architecture, and report results from a live pilot with a global airline, 
evaluated using the synthetic control method.

\subsection{Background and Deployment Challenges}
\label{subsec:ancillary_background}

Ancillary products, such as premium seats, baggage, meals, and in-flight WiFi, have
become an increasingly important revenue stream for airlines, reaching a record
\$117.9 billion globally in 2023 \citep{ideaworks2023}. While airlines have long
employed sophisticated RM systems to dynamically optimize
economy fares \citep{talluri2006theory,phillips2021pricing}, ancillary pricing has
historically relied on static, manually crafted rules, representing a significant
missed opportunity for data-driven and adaptive optimization
\citep{holloway2016straight,McKinsey2017}.

In partnership with a major global airline, we explored how ancillary pricing could
be modernized while respecting operational and regulatory constraints. A central
requirement is compliance with industry-standard \emph{fare filing} rules governed
by the Airline Tariff Publishing Company (ATPCO), through which airlines publish
pricing logic to global distribution systems and online travel agencies such as
Expedia and Travelocity \citep{APTCO_pricing}.

The ATPCO framework supports a standardized but limited set of attributes for
defining pricing logic, including origin--destination pairs, booking windows,
departure time bands, fare restrictions, and stay duration limits
\citep{APTCO_farerule}. As a result, ancillary pricing policies must be expressed
as transparent, rule-based logic over a constrained feature space. A typical rule
takes the form:
\begin{quote}
\emph{``If the flight is between Airports A and B, the ticket was purchased within
7 days of departure, and the departure time is on a weekday before 11 a.m.
$\Rightarrow$ set the upsell price for first-class seats to \$650.''}
\end{quote}

Beyond fare filing, ancillary pricing deployments face three recurring challenges:
\emph{(i)} scalability with interpretability in large-scale, regulated settings;
\emph{(ii)} adherence to complex business constraints, such as monotonicity,
consistency, and fairness across customer segments; and \emph{(iii)} integration
with existing RM infrastructure, where deeply embedded systems make invasive
changes operationally risky and long-term robustness to upstream model changes
essential.

Taken together, these considerations imply that a successful ancillary pricing
solution must be interpretable, constraint-aware, system-compatible, and decoupled
from specific predictive architectures.

\subsection{High-Level Architecture}
\label{subsec:system_architecture}

\begin{figure}[htbp]
    \centering
    \includegraphics[width=0.9\textwidth]{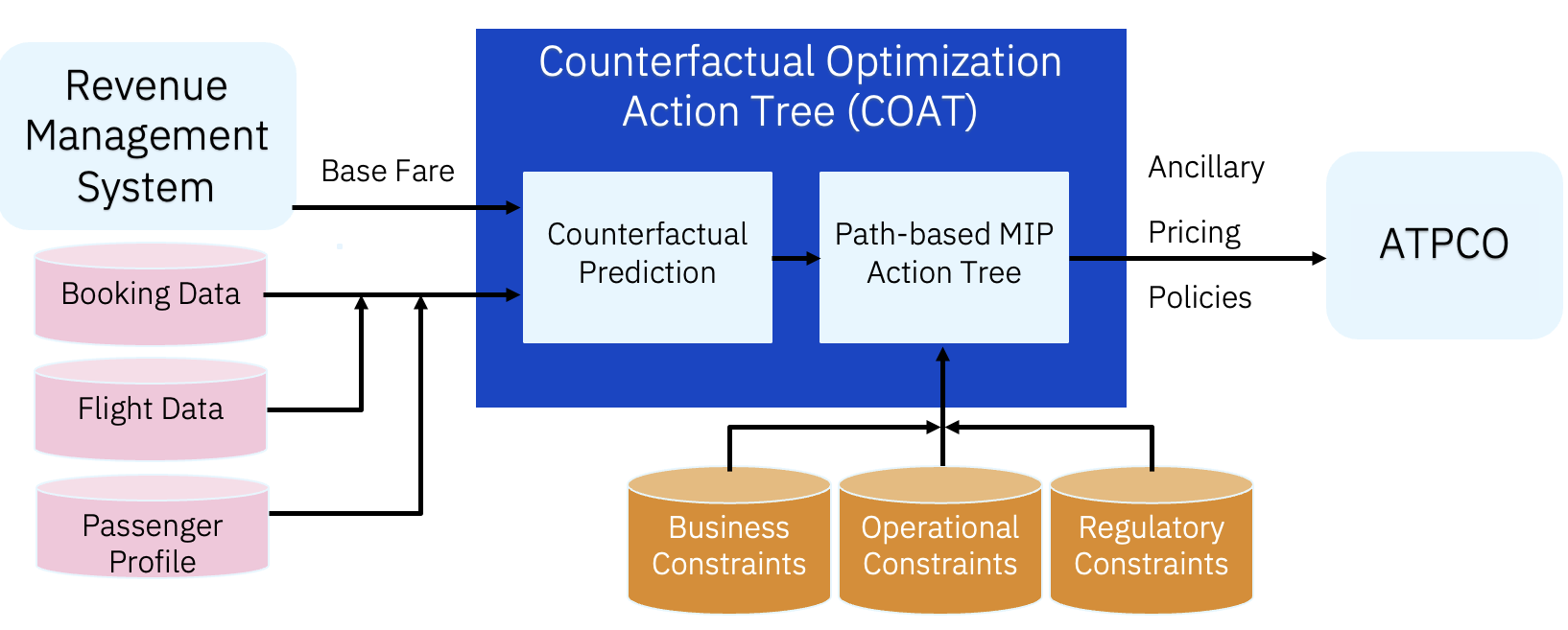}
    \caption{System architecture for airline ancillary pricing optimization layer}
    \label{fig:sytem_overview}
\end{figure}

Figure~\ref{fig:sytem_overview} illustrates the system architecture of the proposed solution. Rather than optimizing absolute ancillary prices, we
treat the economy fare produced by the airline’s RM system as an input and
formulate the decision problem as selecting an optimal \emph{markup} over this
base fare. This design anchors ancillary pricing to the existing fare structure
while preserving core RM workflows.

Formulating the problem in terms of upsell markups enables a modular pricing layer
that can be deployed on top of existing RM infrastructure with minimal engineering
complexity and operational risk. It also ensures that prescribed prices remain
interpretable and consistent with published fare rules, facilitating governance
and ATPCO compliance.

The ancillary pricing policy is learned using the \textsc{COAT} framework
introduced earlier in the paper. Here, \textsc{COAT} acts as a
prescriptive optimization layer that takes counterfactual demand estimates as
inputs and outputs a compact, rule-based pricing policy represented as a multiway
policy tree. The policy depends only on a small set of deployable features,
supporting real-time execution and regulatory compliance, while remaining
agnostic to the upstream predictive model.

We next describe the end-to-end deployment of this architecture for pricing seat
upgrades (e.g., premium economy and first-class seating). The same design extends
naturally to other ancillary products, such as baggage, meals, and in-flight
WiFi, by modifying the feature set and action space.

\subsection{Data}

We had access to several years of transactional booking data spanning early 2019
through the most recent period. For model training, we restrict attention to data
from mid-2022 onward to align with current market conditions, reflecting
post–COVID-19 structural shifts in travel behavior and demand patterns and
ensuring consistency with the pilot deployment period. The resulting training
corpus, spanning mid-2022 through mid-2023, comprises approximately 43 million
individual transaction records and is used to fit the counterfactual outcome and
propensity models. The pilot itself is evaluated on the test and control markets
described in Section~\ref{sect:pilot}.

For the seat-upgrade pricing use case, each observation is a triplet
\( \{X,T,Y\} \), where \( X \) denotes the booking context, \( T \) is the upsell
\emph{markup} (the price difference between the premium seat and the corresponding
economy fare), and \( Y \in \{0,1\} \) indicates whether the customer accepted the
offer.

The feature vector \( X \) integrates information from three sources, providing a
rich representation of the pricing context and supporting demand modeling for
seat upgrades.

\paragraph{Booking-related features}
capture booking timing and channel characteristics, including calendar
indicators, book-to-flight time, trip type, weekend stays, sales channel, premium
seat availability at booking, and the economy fare, which serves as the reference
price for optimizing the markup.

\paragraph{Flight-level features}
describe route and schedule attributes, including origin--destination pairs,
flight distance and duration, departure date and time, time-of-day indicators,
and seasonality variables capturing holidays, weekends, and peak travel periods.

\paragraph{Passenger-level features}
reflect the customer’s relationship with the airline and past behavior, including
loyalty program enrollment, tier status, frequent flyer indicators, and summary
measures of recency, frequency, and cumulative spend.

\subsection{Counterfactual Outcome Estimation}

The first stage of the COAT framework estimates counterfactual outcomes, namely,
the probability that a customer would purchase an upgrade if offered an
alternative markup. 
Because upgrade prices are not randomly assigned, higher prices are often deployed
in contexts with anticipated higher demand, naive predictive models trained on
observational data can suffer from selection bias and may incorrectly infer
upward-sloping demand curves. Correcting for this bias is essential for reliable
pricing decisions.

To address this challenge, we adopt a \emph{doubly robust (DR) estimator}
\citep{dudik2011doubly}, which combines outcome modeling with propensity score
estimation and yields consistent estimates if either component is correctly
specified. The outcome model estimates the probability of upgrade acceptance at a given
markup, while the propensity model estimates the probability of observing that
markup given the booking context.  
Both the outcome model and the propensity model were trained using an automated ML pipeline for algorithm selection, feature engineering, and
hyperparameter tuning. In this application, gradient-boosted decision trees
(XGBoost) \citep{chen2016xgboost} were selected for both components.


The feasible markup grid for each market was defined based on historical price
distributions and business input. Long-haul routes used coarser increments (e.g.,
\$5), while short-haul routes used finer granularity. Across markets, the grid
typically consisted of 10 to 25 discrete price points. During live deployment, the
grid was periodically refined to introduce new price points and improve learning
of price-response behavior.

Once trained, the counterfactual model was used to generate predictions
\( g_{i,t} \), representing the estimated probability that booking \( i \) would
accept an upgrade at price point \( t \in [m] \) defined in the price grid. These predictions form the sole
inputs to the downstream prescriptive optimization model.


In markets with limited historical price variation, raw model predictions occasionally produced economically implausible price-response curves, such as locally positive price elasticities. Before passing the predictions to the optimization layer, we therefore applied a lightweight calibration step. For each observed booking context, we fixed non-price covariates, evaluated the model on the feasible price grid, and adjusted the resulting response curve to satisfy economically plausible elasticity bounds \citep{pindyck2018microeconomics}. The calibrated curve was anchored at the factual observation and informed by analyst-provided elasticity ranges and standard demand-modeling principles. The resultant calibration task is a low-dimensional curve-fitting problem in the space of the price covariates that can be executed in seconds. This transformation preserves the modular structure of COAT while improving the suitability of the counterfactual inputs for price optimization.

A related practical challenge arose when new prices introduced during the pilot had zero
historical support, rendering inverse propensity weights undefined. After
evaluating several alternatives, including propensity smoothing, we adopted an
outcome-only modeling strategy for these cases due to its robustness and
implementation simplicity. Additional discussion is provided in
Section~\ref{appendix:DR_unseen_price}.


\subsection{Policy Graph and Constraint Modeling for Ancillary Pricing}

Regulatory and operational constraints require ancillary pricing policies to
depend on a restricted, interpretable feature set, which in turn defines the
\emph{policy graph} over which prescriptive decisions are optimized. While a rich
set of features is used for counterfactual estimation, the prescriptive policy is
limited to attributes compatible with ATPCO fare filing rules and real-time
execution. The final feature set was determined in collaboration with the airline
and includes origin--destination pairs, flight timing, book-to-flight days, fare
class, and trip type, i.e., features that are operationally accessible and suitable for
publication. Following extensive experimentation, the airline imposed a
user-defined upper bound of 50 rules per market, reflecting a tradeoff between revenue
optimality and policy complexity for governance and auditability.

\paragraph{Capacity constraints.}
Capacity constraints are modeled as inter-rule constraints. Ideally, such
constraints would be formulated using real-time inventory or load-factor
information. However, because inventory data were not available at decision time
and pricing policies were refreshed weekly, we instead imposed \emph{surrogate
constraints} based on predicted outcomes.

Specifically, for each origin--destination pair and product, we constrained the
aggregate predicted upgrade conversion rate under the prescribed prices to remain
within a tolerance band around historical levels. These constraints serve as
stability guardrails, preventing abrupt demand shifts while still allowing
meaningful price differentiation across customer segments.

\paragraph{Domain-specific pricing logic.}
Additional business rules, such as monotonicity with respect to booking time
(e.g., higher prices closer to departure), arbitrage prevention (e.g., ensuring
that two one-way upgrades do not undercut a round-trip upgrade), and systematic
price differentiation for itineraries without weekend stays, which are typically
associated with business travel—are also modeled as inter-rule constraints.

\paragraph{Guardrail constraints.}
To control the risk of systematic overpricing or underpricing, we integrate
pricing recall metrics directly into the optimization as \emph{guardrail
constraints}. We adopt the price-decrease recall (PDR) and price-increase recall
(PIR) metrics introduced by \citet{airbnb}, which provide interpretable proxies
for pricing quality in settings where ground-truth optimal prices are unavailable.
Whereas prior work typically uses these metrics for post-hoc evaluation, we
incorporate them explicitly during policy optimization to regulate aggregate
pricing behavior.

Intuitively, PDR and PIR measure how frequently the prescribed prices decrease or
increase relative to historical prices across observed transactions, including
both upgrade and non-upgrade outcomes. Bounding these quantities limits excessive
upward or downward price movements in aggregate, helping mitigate revenue risk,
customer dissatisfaction, and fairness concerns arising from systematic price
shifts.

Let \( z_j \) indicate whether rule \( j \) is selected, \( \delta_j \) denote the
number of observations covered by rule \( j \) for which the prescribed price is
lower than the historical price, and \( \eta_j \) denote the total number of
observations covered by that rule. The aggregate PDR constraint is given by
\begin{equation}
\label{pdr_constraint}
\pdr_{\mathrm{lb}}
\;\le\;
\frac{\sum_j \delta_j z_j}{\sum_j \eta_j z_j}
\;\le\;
\pdr_{\mathrm{ub}}.
\end{equation}

This constraint anchors the learned policy to historical pricing behavior in
aggregate while still allowing targeted deviations where supported by the data.
As shown in Section~\ref{pdr}, the constraint can be reformulated as a
mixed-integer linear constraint using a Charnes--Cooper transformation and
enforced either globally in the master problem or locally during path generation.

\paragraph{Rule-validity constraints.}
Rule-validity constraints ensure interpretability, statistical reliability, and
compliance with airline pricing practices. These include limits on rule length,
prohibitions on disallowed feature combinations, and minimum sample coverage
requirements. Because these constraints apply to individual paths rather than
collections of rules, they are modeled as intra-rule constraints and enforced
directly during path construction in the policy graph. Infeasible paths are pruned
early in the pricing subproblem, avoiding the need to linearize complex logical
constraints in the master problem.

\paragraph{Forward-looking customer mix.}
Because the booking mix varies over time, we adjust the sample weights used in policy optimization to better reflect the expected customer mix during deployment. In implementation, path-level sample counts are updated using a lightweight exponential-smoothing forecast based on recent booking patterns. This allows the policy to optimize against a forward-looking distribution rather than the historical sample distribution alone.

As discussed in Section~\ref{sect_constraint}, this separation between inter-rule
and intra-rule constraints is critical for scalability. Aggregate business logic
is handled through dual prices in the master problem, while rule-level feasibility
is enforced locally during column generation. This design enables the generation
of compact, interpretable, and operationally valid pricing policies suitable for
enterprise deployment.

\subsection{Pilot Design}
\label{sect:pilot}

We conducted a live field pilot with a major global airline in 2023 to evaluate
whether AI-driven, interpretable pricing policies could deliver measurable
business impact while operating under real-world regulatory, fairness, and
operational constraints. The pilot was initially planned as a 12-week engagement,
with pricing policies refreshed weekly and models retrained using the most recent
transaction data. Encouraged by early results, the pilot was extended and
ultimately transitioned into sustained operation.

This section describes the experimental design, evaluation methodology, and
deployment outcomes. Particular attention is paid to causal identification and
fairness considerations, which play a central role in airline pricing and
preclude the use of standard randomized experimentation.

\subsubsection{Evaluation Metrics}

The primary business metric used to evaluate pricing performance is
\emph{upsell revenue per booking}, defined as total premium ancillary revenue
normalized by the number of bookings:
\begin{equation}
\text{Upsell Revenue per Booking}
=
\frac{\text{Total Upsell Revenue}}{\text{Number of Bookings}}.
\end{equation}

This normalization is essential in airline settings, where booking volumes vary
substantially over time due to seasonality, demand shocks, and macroeconomic
factors. Measuring revenue on a per-booking basis enables meaningful comparisons
across time periods and markets, independent of traffic fluctuations. Conversion
rate (number of upgrades sold) and yield (revenue per upgrade) are tracked as
secondary metrics to help interpret the drivers of revenue changes.

\subsubsection{Evaluation Design and Identification}

In principle, the gold standard for estimating treatment effects is a randomized
controlled trial  or A/B test. In airline pricing applications, however, such
experiments are often infeasible. Offering different prices to customers who are
otherwise indistinguishable based on observable booking context (e.g.,
origin--destination, fare class, travel dates, loyalty status) raises legal,
regulatory, and ethical concerns related to price discrimination and fairness.

Several alternative designs were considered. Temporal alternating
between AI pricing and legacy pricing was rejected due to the advance-purchase
nature of airline bookings, which makes it impossible to cleanly align treatment
timing with observed outcomes. Year-over-year comparisons were also unsuitable,
as demand patterns during 2020--2022 were heavily distorted by the COVID-19
pandemic and do not provide a stable baseline.

Given these constraints, the airline adopted a \emph{test-versus-control market}
design. Twelve test markets were selected across short-, medium-, and long-haul
routes. For each test market, three to four control markets were identified based
on geographic similarity, historical booking volumes, ancillary attach rates, and
expected demand stability, yielding 30 distinct control markets. While this design provides a reasonable baseline,
direct comparisons remain vulnerable to time-varying confounders such as
seasonality and demand shifts, motivating the use of a more rigorous causal
evaluation method.

\subsubsection{Synthetic Control Method}
\label{sect:scm}

To isolate the causal effect of the AI-driven pricing policy, we employ the
\emph{synthetic control method} (SCM)~\citep{abadie2010synthetic}. As
\citet{athey2017state} note, \emph{``arguably the most important innovation in the
evaluation literature in the last 15 years is the synthetic control approach''},
particularly in settings where randomized experimentation is infeasible.

SCM constructs a data-driven counterfactual for each treated market using a
weighted combination of control markets, with weights chosen to closely match the
treated market’s outcome trajectory during a pre-pilot period. By explicitly
aligning pre-intervention trends, SCM provides a more credible benchmark than
simple averaging across matched controls, especially in environments with strong
seasonality and evolving demand conditions. More details can be found in Section~\ref{ec_SCM}. 

In our deployment, SCM weights were estimated using a nine-week pre-pilot window.
Figure~\ref{fig_SCM_weights} illustrates the resulting weight distributions, which
are highly non-uniform across control markets. This heterogeneity reflects the
importance of data-driven weighting in reproducing historical dynamics and
highlights the limitations of nominal control averages. By anchoring inference on
observed pre-treatment behavior, SCM enables a rigorous and interpretable
assessment of the pricing policy’s causal impact in a real-world operational
setting.

\begin{figure}
 \centering
  \includegraphics[width=.6\linewidth]{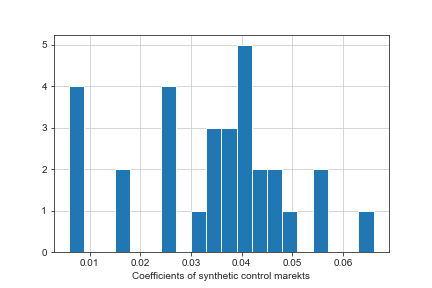}
  \caption{Histogram of synthetic control weights across matched control markets.}
  \label{fig_SCM_weights}
\end{figure}

\subsection{Pilot Results}
\label{sec:pilot_results}

To evaluate the performance of the AI-driven pricing pilot, we align all test and
control markets on a common weekly timeline. Week~1 corresponds to the start of
the pre-pilot observation period in 2023, and the intervention begins in Week~10,
when the prescriptive tree policy was deployed in production. From that point
onward, upsell prices in the 12 designated test markets were determined by
AI-generated policies, updated on a weekly cadence.

Although the pilot was initially planned to run for 12 weeks, strong business
performance led the airline to extend the deployment. We report results observed
through Week~26.

\subsubsection{Quantitative Impact on Revenue}

We evaluate the revenue impact of the AI-driven pricing policy using the
synthetic control method (SCM). Control weights are estimated using the
pre-pilot period (Weeks~1--9), and post-intervention effects are assessed by
comparing the treatment markets to their synthetic counterparts from Week~10
onward. For reference, we also report a nominal control benchmark constructed
by uniformly averaging the same control markets.

\begin{figure}
 \centering
  \includegraphics[width=.8\linewidth]
  {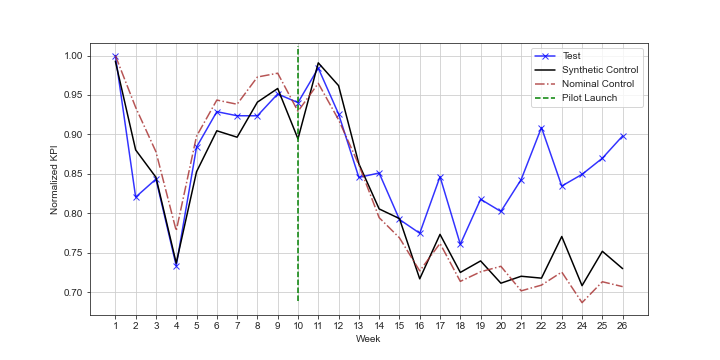}
  \caption{Normalized weekly KPI trajectories for the treatment group under \textsc{COAT}, synthetic control, and nominal control. The vertical dashed line marks the pilot launch.}
  \label{fig:scm_comparison}
\end{figure}

Figure~\ref{fig:scm_comparison} shows the resulting KPI trajectories over the
26-week evaluation horizon. The synthetic control closely tracks the
pre-intervention trend of the treatment group, whereas the nominal control
exhibits larger deviations prior to deployment. Following the intervention,
the treatment group consistently outperforms both benchmarks, indicating a
positive revenue impact of the AI-driven pricing policy.

The difference in pre-pilot fit between the two benchmarks is economically
meaningful. Because the synthetic control lies above the nominal control for
much of the post-pilot period, reliance on uniform averaging would have
overstated the treatment effect. This comparison highlights the importance of
using a well-matched counterfactual in observational field settings.

Figure~\ref{fig_SCM_treatment} reports the estimated treatment effect, defined
as the difference between the treatment KPI and its synthetic control. We
summarize it as the average weekly gap over each window. Early gains were
negligible, averaging $0.5\%$ during Weeks~10--15. The impact strengthened over
time, averaging $8.9\%$ during Weeks~16--26. Over the full post-intervention
period (Weeks~10--26), the average weekly treatment effect is $6.9\%$ relative to
the synthetic control.

This increasing effect reflects an operational learning dynamic: early in
deployment, limited historical price variation constrained inference of
fine-grained demand responses. As the pricing policy was retrained weekly and
the price grid was progressively refined, richer variation improved learning
and enabled more precise price differentiation across segments.

Placebo-based SCM tests support the robustness of these results, rejecting the
null of no effect at the 5\% level ($p = 0.032$); details are provided in
Section~\ref{appendix_placebo}.

\begin{figure}
 \centering
  \includegraphics[width=.8\linewidth]
  {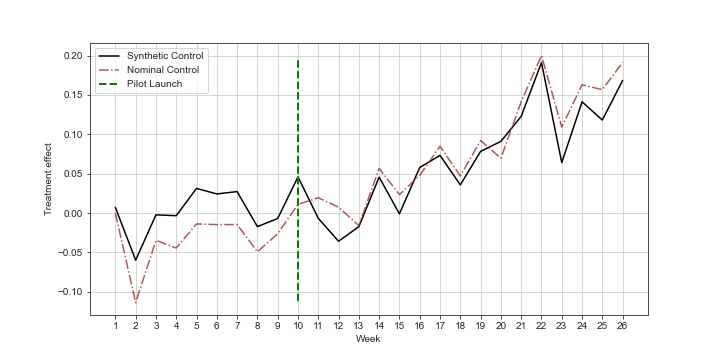}
  \caption{Estimated weekly treatment effect under \textsc{COAT} relative to the synthetic control.}
  \label{fig_SCM_treatment}
\end{figure}

\subsubsection{Behavioral and Fairness Analysis}

Beyond aggregate revenue gains, a central concern for the airline was whether
AI-generated prices would remain stable, interpretable, and fair. To assess these
dimensions, we examine how recommended prices vary across booking segments
relative to historical baselines.

Figure~\ref{fig:scm_prices} visualizes AI-recommended prices at the segment level.
Each bubble represents a segment (scaled by booking volume), with the horizontal
axis indicating the historical conversion rate and the vertical axis showing the
percentage change in recommended price relative to legacy pricing. The policy
systematically increases prices for high-conversion, price-inelastic segments
and decreases prices for low-conversion, price-elastic segments.

Crucially, this behavior arises under the explicit \emph{guardrail constraints}
 embedded in the optimization model. These inter-rule constraints
bound the aggregate frequency of price increases and decreases relative to
historical pricing, anchoring the policy to observed market behavior and
preventing systematic overpricing or underpricing at the market level. As a
result, the policy adapts to heterogeneous willingness to pay while avoiding
abrupt pricing shocks or discriminatory patterns.

Notably, the learned policy recommends more price decreases than increases in aggregate, directly benefiting price-sensitive customers. By contrast, the flat legacy upsell prices were often less attractive to price-sensitive leisure travelers, while leaving revenue unrealized among late-booking business travelers, who typically exhibit higher willingness to pay. Moreover, the policy established a baseline upsell increment that allowed the airline to continue applying its inventory-driven operational adjustments without conflict. These gains were achieved with minimal changes to the existing business process and without requiring extensive changes to core revenue-management infrastructure.

\begin{figure}
 \centering
  \includegraphics[width=.7\linewidth]
  {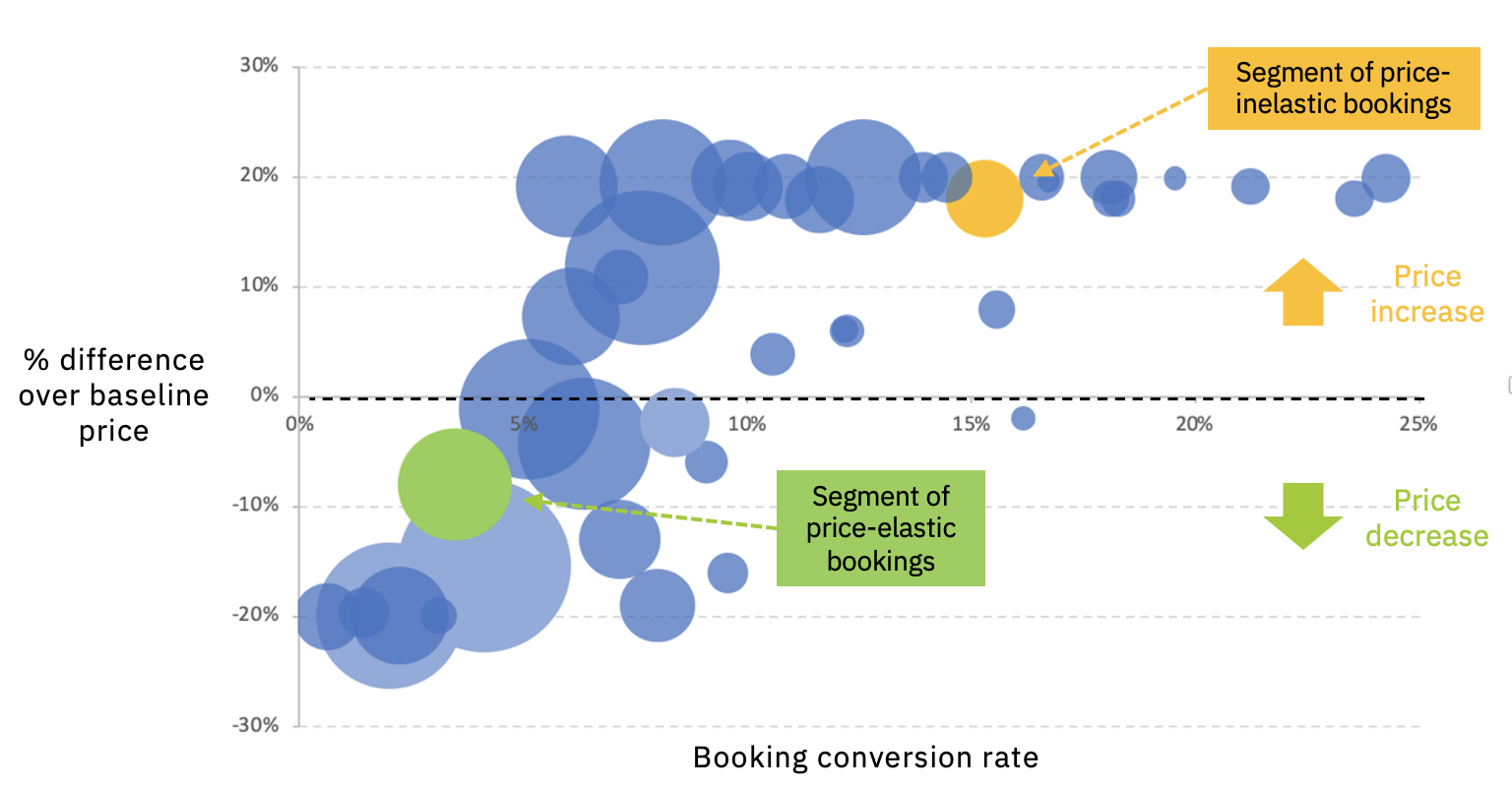}
  \caption{Distribution of AI-recommended prices relative to historical benchmarks}
  \label{fig:scm_prices}
\end{figure}

\subsubsection{Secondary Metrics and Operational Implications}

Table~\ref{tab:pilot_results} reports secondary performance metrics, including
conversion rate and yield. While the AI-driven pricing policy led to a modest
decline in upgrade conversion (–0.6\%), it delivered a substantial increase in
yield (+16.4\%), indicating improved monetization efficiency.

This pattern reflects intentional tradeoffs. By raising prices for
price-insensitive segments and lowering prices elsewhere, the model sells fewer
upgrades overall but at more appropriate price points. From an operational
perspective, the reduction in paid upgrades was viewed positively: it preserved
premium inventory for high-value uses such as complimentary upgrades for loyal
customers, supporting broader customer experience objectives.

Taken together, these results demonstrate that the AI-driven pricing policy
achieves revenue gains not through indiscriminate price increases, but through
controlled, interpretable price differentiation that respects fairness and
operational constraints.

\begin{table}[htbp]
\centering
\caption{Relative performance of test markets compared to control markets.}
\label{tab:pilot_results}
\begin{tabular}{lccc}
\toprule
& Overall & \multicolumn{2}{c}{By sub-period} \\
\cmidrule(lr){2-2} \cmidrule(lr){3-4}
Metric (\%) & Weeks 10--26 & Weeks 10--15 & Weeks 16--26 \\
\midrule
Revenue per booking & 6.9 & 0.5 & 8.9 \\
Conversion rate  & -0.6  & -0.6 & -0.7 \\
Yield (Revenue per upgrade) & 16.4 & 11.9 & 20.5 \\
\bottomrule
\end{tabular}
\end{table}

\subsubsection{Organizational Adoption and Strategic Impact}

Encouraged by early results, the initial 12-week pilot was extended by an additional five weeks and ultimately transitioned into continuous operation. The system has remained in production use since the pilot, indicating durability well beyond the initial evaluation window. Based on observed performance during the pilot, the airline projected that scaling the approach across eligible domestic markets could generate approximately \$50--\$150 million in incremental annual premium seat revenue, substantially exceeding the original success threshold of a 2--3\% uplift. The airline further noted that the revenue potential is likely higher in
international markets, where ancillary markups are typically larger. As a
result, the COAT framework was formally adopted as a strategic component of the
airline’s long-term ancillary pricing strategy.

Beyond financial impact, the pilot marked a shift in internal perceptions of AI,
from an exploratory capability to a reliable driver of measurable business
value. As one senior business sponsor remarked:
\vspace{0.5em}
\begin{quote}
\emph{``This is the first success we have had at the airline in getting AI
to deliver measurable business results.''}
\end{quote}
\vspace{0.5em}

This shift was enabled by the framework’s emphasis on interpretability and
operational alignment. By integrating counterfactual inference with large-scale
optimization, \textsc{COAT} produced pricing policies that respected regulatory and
business constraints while remaining transparent and auditable. The rule-based
structure allowed pricing analysts and business stakeholders to inspect,
validate, and adjust decisions, fostering cross-functional trust and easing
organizational adoption.
Taken together, the pilot demonstrates how a rigorously formulated prescriptive
optimization framework can translate causal modeling advances into sustained,
enterprise-scale impact. 

\subsection{Managerial Insights}

This study offers several lessons for managers deploying AI in high-stakes operational settings, where decisions must balance performance, fairness, interpretability, and operational constraints.

\paragraph{From prediction to prescription.}
Many enterprise AI initiatives emphasize predictive accuracy, such as forecasting demand, conversion, or willingness to pay. However, business value is realized only when predictions are translated into decisions that respect operational constraints. The COAT framework makes the decision problem explicit: it learns counterfactual outcomes and optimizes over a constrained action space, thereby targeting decision quality rather than intermediate predictive metrics. In the airline pricing application, this shift enabled measurable revenue gains while preserving managerial control over when and how prices change.

\paragraph{Interpretability and constraints can enable adoption.}
In regulated and customer-facing domains, interpretability and compliance are often viewed as obstacles to AI-driven optimization. Our deployment suggests the opposite. Representing policies as interpretable prescriptive trees and embedding business and regulatory constraints made the system easier to audit, validate, and deploy. These features were central to addressing governance concerns and external scrutiny around pricing fairness. Rather than limiting performance, interpretability and constraints helped build organizational trust and enabled sustained adoption.

\paragraph{Modular deployment reduces risk but limits upside.}
The deployed system optimized ancillary prices as markups over existing economy fares, treating the airline's revenue management system as an upstream input rather than pursuing full end-to-end integration. This modular design reduced engineering complexity and operational risk, allowing rapid deployment without disrupting legacy infrastructure. At the same time, modularity limits responsiveness because the pricing layer does not directly incorporate real-time capacity and demand signals, such as remaining seat inventory or booking dynamics. A promising direction for future operational gains is to integrate prescriptive decision layers more tightly with core operational systems while preserving auditability and control.

\paragraph{Credible evaluation and fairness controls are essential.}
Evaluating AI-driven decisions is difficult when randomized experimentation is infeasible due to legal, ethical, or operational constraints. In this setting, rigorous causal evaluation methods, including synthetic control, provided credible estimates of the deployed policy's incremental impact and helped build stakeholder confidence. Equally important, fairness-oriented guardrails within the optimization model helped ensure that revenue gains were not driven by systematic overpricing or unintended distributional effects. For managers, this highlights the importance of pairing prescriptive AI systems with both credible evaluation frameworks and explicit fairness controls.

\paragraph{AI adoption requires a decision layer.}
The broader managerial lesson is that successful AI adoption is not simply a matter of improving prediction or replacing existing decision processes. Modern machine learning can expand predictive capabilities, but organizations still need principled ways to optimize decisions under uncertainty, constraints, and governance requirements. COAT illustrates how operations research can provide this decision layer, translating AI advances into deployable, auditable, and high-impact operational systems.

\section{Conclusion}

This paper introduced the Counterfactual Optimal Action Tree (COAT) framework, which integrates causal inference and large-scale optimization to learn interpretable decision policies from observational data. By decoupling counterfactual outcome estimation from policy optimization, COAT constructs decision rules that are data-driven, interpretable, and explicitly aware of operational and regulatory constraints. In a large-scale airline pricing deployment, the framework delivered statistically and economically significant revenue gains, demonstrating the feasibility of prescriptive AI at enterprise scale.

Beyond the specific application, the results highlight a broader opportunity for operations research in AI-enabled decision-making. While modern machine learning excels at prediction, translating predictive insights into actionable, auditable, and compliant decisions remains a central challenge in practice. COAT shows how OR methods can bridge this gap by providing a structured optimization layer that aligns AI-driven decisions with business objectives and real-world constraints.

Several directions for future research remain. One is to extend COAT to multi-period settings that capture intertemporal tradeoffs, learning effects, and demand dynamics. Another is to incorporate uncertainty in counterfactual estimates, model misspecification, or distributional shifts directly into the optimization stage. These extensions would further strengthen the role of structured, optimization-based decision frameworks in supporting reliable and resilient AI deployments.

\bibliographystyle{informs2014} 
\bibliography{ref} 

\newcounter{mainlastpage}
\setcounter{mainlastpage}{\value{page}}
\ECSwitch

\ECHead{Electronic Companion: Counterfactual Optimal Action Trees (COAT)}

\section{Additional Algorithmic Details on Handling Numerical Features}
\label{subsect_numerical_features}

This section supplements Section~\ref{sect_rule_space} of the main paper and
describes how numerical features are incorporated into the \textsc{COAT}
framework.

Tree-based methods typically require numerical features to be discretized before
policy learning. A common approach is to partition a continuous feature into a
fixed set of disjoint intervals. For example, consider a numerical feature with
support on \([0,1]\) discretized into \(\kappa=3\) bins,
\([0,0.33)\), \([0.33,0.67)\), and \([0.67,1.0)\). In a multiway-split tree, this
corresponds to creating three mutually exclusive nodes, one per interval.

While simple, this representation can be restrictive. In particular, it cannot
capture threshold-based decisions commonly expressed in binary trees, such as
\(x \le 0.67\) or \(x > 0.33\), without introducing additional tree depth. To
address this limitation, we adopt a \emph{cumulative binning} (CB) scheme that
allows overlapping intervals.

Under cumulative binning, in addition to the original disjoint bins, we create
nodes corresponding to unions of adjacent intervals. In the above example, this
yields the additional intervals \([0,0.67)\), \([0.33,1.0)\), and \([0,1.0)\),
resulting in a total of six candidate nodes. This construction enables
multiway-split trees to represent threshold-type rules within a single level,
thereby improving expressiveness without increasing tree depth. We employed
cumulative binning in the live airline pilot to capture operationally meaningful
pricing thresholds while preserving a shallow and interpretable policy structure.

The cumulative binning approach generalizes naturally: discretizing a numerical
feature into \(\kappa\) base bins produces \(\kappa(\kappa+1)/2\) cumulative
intervals. While this increases the size of the policy graph, the resulting growth
in the policy space is handled effectively by the column generation algorithm,
which dynamically introduces only those paths that improve the objective.

Finally, we note that the \textsc{COAT} formulation enforces a coverage constraint
to ensure that the selected rules do not overlap in their sample assignments.
This guarantees that, despite the presence of overlapping bins, the final
prescriptive policy remains well-defined and assigns a unique action to each
observation.

\section{Exact State Aggregation in Policy Prescription}
\label{ec:state_aggregation}

This section provides the formal details underlying the exact state aggregation
described in Section~5.2 of the main paper. In particular, we present the reduced
mixed-integer formulation induced by aggregating decision-equivalent observations
and establish its equivalence to the original \textsc{COAT} formulation.

\subsection{Decision-Equivalent State Aggregation}

Let $X_i \in \mathbb{R}^k$ denote the full feature vector for observation $i$, and
let $Z_i = R(X_i) \in \Omega$ denote the prescriptive feature vector obtained by
restricting $X_i$ to features available at decision time. Observations with
identical prescriptive features are said to be \emph{decision-equivalent}.

Formally, define the equivalence relation
\[
i \sim i' \quad \Longleftrightarrow \quad Z_i = Z_{i'} ,
\]
which partitions the dataset into $N'$ equivalence classes
\[
\mathcal{G}_1, \dots, \mathcal{G}_{N'}, \qquad
N' = \bigl| \{ Z_i : i = 1,\dots,n \} \bigr|.
\]

Because all prescriptive policy rules depend only on $Z$, every observation within
a group $\mathcal{G}_{i'}$ satisfies exactly the same set of candidate paths in the
policy graph. Consequently, coverage indicators are constant within each group.

\subsection{Reduced Mixed-Integer Formulation}
\label{ec:reduced_mip}

Let $a_{ij}$ denote the coverage indicator specifying whether path $j$ covers
observation $i$ in the original formulation. For each group $\mathcal{G}_{i'}$,
define
\[
a_{i'j} := a_{ij} \quad \text{for any } i \in \mathcal{G}_{i'} .
\]

Define the aggregated penalty and reward for group $i'$ under path $j$ as
\[
c_{i'} := \sum_{i \in \mathcal{G}_{i'}} c_i,
\qquad
r_{i'j} := \sum_{i \in \mathcal{G}_{i'}} g_{i,t_j},
\]
where $g_{i,t_j}$ denotes the counterfactual reward of assigning action $t_j$ to
observation $i$.

Using these grouped quantities, the \textsc{COAT} problem admits the following
exact reformulation:
\begin{align}
\textsc{(Reduced--COAT)} \quad
\max \quad &
    \sum_{j=1}^{|\mathcal{P}|} r_j z_j
    - \sum_{i'=1}^{N'} c_{i'} s_{i'} \nonumber \\
\text{s.t.} \quad
& \sum_{j=1}^{|\mathcal{P}|} a_{i'j} z_j + s_{i'} = 1,
    && \forall i' = 1,\dots,N', \nonumber \\
& \sum_{j=1}^{|\mathcal{P}|} z_j \le l, \quad z \in \mathcal{Z}, \nonumber \\
& z_j \in \{0,1\}, \quad s_{i'} \ge 0. \nonumber
\end{align}

Because all observations within $\mathcal{G}_{i'}$ share identical coverage
constraints, replacing individual constraints with a single group-level
constraint is exact. The reduced formulation therefore preserves feasibility and
objective value for every feasible policy.

\subsection{Equivalence to the Original \textsc{COAT} Formulation}

First, for any feasible solution $(z,s)$ to the original \textsc{COAT}
formulation, define $s_{i'} := \sum_{i \in \mathcal{G}_{i'}} s_i$. Because coverage
indicators are constant within each group, the group-level coverage constraints
are satisfied exactly, and the objective value is preserved by linearity.

Conversely, given any feasible solution $(z,s')$ to the reduced formulation, one
can construct a feasible solution to the original formulation by assigning the
same slack $s_{i'}$ to each observation $i \in \mathcal{G}_{i'}$. Feasibility and
objective value are again preserved.

Therefore, the original and reduced formulations are equivalent: they admit the
same optimal objective value and the same set of optimal policy selections $z$.

\section{Practical Adaptation of Counterfactual Estimation for Expanding Action Spaces}
\label{appendix:DR_unseen_price}

A key challenge in our setting is that not all upsell prices are historically observed, particularly as the feasible price grid is expanded during deployment. Standard doubly robust (DR) estimators combine an outcome regression with inverse propensity weighting (IPW) to achieve consistency when either component is correctly specified. For a given price \( t \) and covariates \( X_i \), the DR estimator takes the form
\[
\hat{Y}^{\mathrm{DR}}_i(t)
= \hat{f}(X_i,t)
+ \frac{\mathbb{I}(T_i=t)}{\hat{e}(X_i,t)}
\bigl(Y_i - \hat{f}(X_i,t)\bigr),
\]
where \( \hat{f}(X_i,t) \) denotes the outcome model and
\( \hat{e}(X_i,t)=P(T_i=t\mid X_i) \) the propensity score.  
When \( \hat{e}(X_i,t)=0 \), however, the estimator is undefined. This reflects a structural limitation rather than estimation error: inverse propensity weighting cannot be applied to actions outside the historical support. Consequently, standard DR methods cannot be directly used when policies are allowed to select prices not previously observed.

We considered two approaches to address this issue.

\paragraph{(1) Outcome-only modeling.}
This approach omits the IPW correction and relies solely on the regression model to impute counterfactual outcomes across the feasible price grid. For each observation \( i \),
\[
\hat{Y}_i(t)=\hat{f}(X_i,t), \quad \forall t\in\mathcal{T}_i,
\]
where \( \mathcal{T}_i \subseteq \mathcal{T} \) denotes the set of admissible prices for that instance. These estimates are used directly in the downstream policy optimization problem. While this sacrifices the double-robustness guarantee, it avoids undefined estimators and permits extrapolation to unseen prices, provided the outcome model generalizes appropriately.

\paragraph{(2) Propensity smoothing.}
As an alternative, we explored smoothing the propensity function over continuous prices to retain a DR-style correction. Specifically, we approximate
\[
\hat{e}(X_i,t)
\approx
\sum_{t'\in\mathcal{T}_{\mathrm{obs}}}
K_\sigma(t,t')\,\hat{e}(X_i,t'),
\]
where \( K_\sigma \) is a smoothing kernel and
\( \mathcal{T}_{\mathrm{obs}} \) denotes observed prices. While this reduces the incidence of zero propensities, it imposes strong smoothness assumptions on pricing behavior and can substantially increase variance when historical support is sparse.

\paragraph{Final implementation.}
Given operational requirements for stability, scalability, and frequent retraining, we adopted the outcome-only approach in production. Models were retrained weekly, progressively expanding the observed support over \( \mathcal{T} \) and improving generalization to new price points. In this setting, outcome-only estimation provided a stable and interpretable basis for policy evaluation while accommodating an expanding action space.

\paragraph{Relation to the identification assumptions.}
We note a tension with the identification assumptions of Section~\ref{sect:counterfactual_estimation}: prices newly introduced during the pilot have zero historical support, so the positivity condition \( P(T=t\mid X=x)>0 \) fails for precisely the actions the policy may select, and the doubly robust correction is undefined there. For these actions we necessarily rely on the outcome model's extrapolation rather than on a consistency guarantee. We mitigate the resulting risk in three ways. First, weekly retraining progressively expands the observed support, restoring overlap for prices that enter the live grid and thereby re-establishing positivity for those actions over time. Second, the guardrail (PDR/PIR) constraints bound the aggregate deviation of prescribed prices from historical pricing, limiting the policy's reliance on poorly-supported regions of the action space. Third, our headline evaluation is based on the synthetic-control analysis of \emph{realized} revenue (Section~\ref{sec:pilot_results}), which measures the deployed policy's causal impact directly and is therefore agnostic to extrapolation error in the upstream estimator. Tighter handling of expanding action spaces, for example via bounds on the extrapolation bias or constrained exploration that guarantees support before deployment, is an important direction for future work.

\section{Proofs of Theoretical Results}
\subsection{Proof of Theorem~\ref{them:equivalence}}
\proof{Proof.}
Each path \( p \in \mathcal{P} \) corresponds to a conjunction of feature-value
tests arranged in a fixed global feature order, where \textsc{SKIP} nodes
indicate omitted features. Consequently, every path represents a valid rule that
assigns an action to all samples satisfying its conditions.

At optimality, the coverage constraint~\eqref{constraint_coverage} ensures that
the selected paths \( \mathcal{P}(\mathbf{z}^*) \) are mutually exclusive and
collectively exhaustive over the sample set, with each sample covered by exactly
one active path.

We construct a multiway decision tree by merging the selected paths into a trie
ordered by the global feature sequence, sharing common prefixes and branching on
feature values whenever paths diverge. Because each path contains at most one
test per feature and follows the same ordering, this construction yields a
valid multiway decision tree with one leaf per selected path and at most
\( l \) leaves.

Finally, since each sample is assigned to exactly one leaf and each leaf
corresponds to a selected path, the total reward of the constructed tree equals
the sum of rewards of the active paths, which is exactly the objective value of
the \textsc{COAT} solution \( \mathbf{z}^* \). 
\endproof

\subsection{Proof of Proposition~\ref{prop_feature_space}}

\proof{Proof. }
Each of the \(k\) features can independently assume up to \(\eta\) values. Any
unconstrained rule may involve any combination of feature assignments, and thus
the number of distinct feature--value combinations is at most \(\eta^k\).
\endproof

\subsection{Proof of Proposition~\ref{prop_constrained_space}.}
\proof{Proof. }
A path of depth \(d\) may involve at most \(d\) features. There are
\(\binom{k}{d}\) ways to select the subset of active features, and for each such
subset there are at most \(\eta^d\) possible value assignments. Hence,
\[
|\mathcal{P}_d| \le \binom{k}{d}\eta^d.
\]
\endproof

\subsection{Proof of Proposition~\ref{prop_lp_upper_bound}}
\proof{Proof. }
By Proposition~\ref{prop_constrained_space},
\( |\mathcal{P}_d| = \mathcal{O}\!\left(\binom{k}{d}\eta^d\right) \), which is
polynomial in \(k\) when \(d\) and \(\eta\) are fixed. The LP relaxation of
\textsc{COAT} over \(\mathcal{P}_d\) therefore has polynomially many variables and
\(n\) coverage constraints, and can be solved in polynomial time using standard
linear programming methods~\citep{khachiyan1980polynomial,karmarkar1984new}.
\endproof

\subsection{Proof of Theorem~\ref{thm:reduction_speedup}}
\proof{Proof. }
With \(k'\) categorical features each taking at most \(\eta\) possible values,
there are at most \(\eta^{k'}\) distinct prescriptive feature combinations,
establishing \(N' \le \eta^{k'}\). Substituting this bound into the ratio
\(n/N'\) yields the stated reduction factor.

For completeness, writing \(n = 2^\delta\) gives the equivalent expression
\[
\frac{n}{N'}
    \ge 2^{\,\delta - k'\log_2 \eta},
\]
which makes explicit the exponential dependence on \(k'\) and \(\eta\). When
\(n < \eta^{k'}\), the bound is trivially at least \(1\), which can be expressed
using the positive-part notation \((x)^+=\max\{x,0\}\) as
\[
\frac{n}{N'}
\ge
2^{(\delta - k'\log_2\eta)^+}.
\]
\endproof

\subsection{Proof of Proposition~\ref{prop:multiway_split_complexity}}
\proof{Proof.}
A binary split on a categorical feature with $\eta$ values corresponds to a
partition of the value set into two nonempty parts, say $S$ and $S^c$.
There are $2^\eta-2$ nonempty proper subsets $S\subset\{1,\dots,\eta\}$, but the
partition $(S,S^c)$ is identical to $(S^c,S)$. Identifying complements therefore
yields $(2^\eta-2)/2 = 2^{\eta-1}-1$ distinct nontrivial binary splits.

In the equality-based representation, admissible tests are exactly
$\{X_f = v : v\in\{1,\dots,\eta\}\}$, hence there are $\eta$ atomic tests.
\endproof

\subsection{Proof of Proposition~\ref{prop:depth_overhead_mip}}
\proof{Proof.}
We first show the upper bound. Consider any multiway node that tests a
categorical feature $X_f\in\{1,\dots,\eta\}$ via equality branches.
This $\eta$-way split can be implemented by a binary decision tree that
identifies the value of $X_f$ using at most $\lceil\log_2 \eta\rceil$ binary
tests, for example by encoding the category index in binary and querying its
bits, or by using a balanced binary search over the $\eta$ labels.
Thus each multiway test contributes at most $\lceil\log_2 \eta\rceil$ to binary
depth. Along any root--to--leaf path with at most $d$ tested features, the total
binary depth is therefore at most $d\lceil\log_2 \eta\rceil$, yielding a binary
tree $\mathcal{T}_B$ that implements the same policy.

We now consider the lower bound. 
Consider the prescriptive domain $\{1,\dots,\eta\}^k$, which contains $\eta^k$
distinct feature profiles. Define a policy that assigns a distinct action (or,
more generally, a distinct leaf label) to each profile.\footnote{If the action
set is smaller than $\eta^k$, one may instead define a policy that still
requires distinguishing all $\eta^k$ profiles, e.g., by mapping each profile to
a distinct leaf label and then merging labels as needed; the depth lower bound
concerns separation of inputs rather than the cardinality of actions.}

Any binary-split decision tree of depth $D$ has at most $2^D$ leaves and thus can
separate at most $2^D$ distinct input profiles into distinct terminal regions.
To implement the policy above, the tree must have at least $\eta^k$ leaves, so
$2^D \ge \eta^k$. Taking $\log_2$ yields $D \ge k\log_2 \eta$, hence
$D \ge \lceil k\log_2 \eta\rceil$.
\endproof

\subsection{Exact MIP reformulation of the PDR constraint}\label{pdr}

This subsection provides algorithmic details for the \emph{guardrail constraints}
introduced in Section~\ref{sect_constraint}. These constraints regulate aggregate
pricing behavior by bounding the frequency of price increases and decreases
relative to historical pricing, as measured by price-increase recall (PIR) and
price-decrease recall (PDR). By constraining these quantities at the policy level,
the optimization mitigates the risk of systematic overpricing or underpricing,
while still allowing targeted deviations supported by the data. We focus here on
the exact mixed-integer linear reformulation of the PDR constraint; the PIR
constraint is handled analogously.

\begin{proposition}[Exact MILP reformulation of the PDR constraint]
\label{prop:pdr_linearization}
Assume at least one rule is selected so that $\sum_{j} \eta_j z_j \ge 1$.
Then the constraint
\begin{equation}
\label{eq:pdr_frac}
\pdr_{\mathrm{lb}}
\;\le\;
\frac{\sum_{j} \delta_j z_j}{\sum_{j} \eta_j z_j}
\;\le\;
\pdr_{\mathrm{ub}}
\end{equation}
admits an exact mixed-integer linear reformulation.
\end{proposition}

\proof{Proof. } 
Let $z_j \in \{0,1\}$ denote whether rule $j$ is selected. Define
\[
N(z) := \sum_{j} \delta_j z_j,
\qquad
D(z) := \sum_{j} \eta_j z_j.
\]
By assumption, $D(z)\ge 1$ for all feasible solutions.

Introduce a continuous variable $w \ge 0$ and auxiliary continuous variables
$y_j \ge 0$ for all $j$, intended to represent
\begin{equation}
\label{eq:cc_def}
w = \frac{1}{D(z)},
\qquad
y_j = w z_j \quad \forall j.
\end{equation}
Enforce \eqref{eq:cc_def} via the linear constraints
\begin{align}
\sum_{j} \eta_j y_j &= 1, \label{eq:cc_norm}\\
0 \le y_j &\le w, && \forall j, \label{eq:rlt1}\\
y_j &\le z_j, && \forall j, \label{eq:rlt2}\\
y_j &\ge w + z_j - 1, && \forall j. \label{eq:rlt3}
\end{align}
Constraints \eqref{eq:rlt1}--\eqref{eq:rlt3} are the standard McCormick/RLT
linearization for the bilinear terms $y_j = w z_j$ when $z_j$ is binary and
$0 \le w \le 1$. (The bound $w \le 1$ follows from $D(z)\ge 1$ and
$w = 1/D(z)$ enforced by \eqref{eq:cc_norm} and nonnegativity.)

We now show exactness. Fix any feasible $(z,w,y)$ satisfying
\eqref{eq:cc_norm}--\eqref{eq:rlt3}.
If $z_j=0$, then \eqref{eq:rlt2} implies $y_j \le 0$, hence $y_j=0$.
If $z_j=1$, then \eqref{eq:rlt1} gives $y_j \le w$ and \eqref{eq:rlt3} gives
$y_j \ge w$, hence $y_j = w$. Therefore, for all $j$,
\[
y_j = w z_j \quad \text{exactly}.
\]
Plugging this identity into \eqref{eq:cc_norm} yields
\[
\sum_{j} \eta_j (w z_j) = 1
\quad\Longleftrightarrow\quad
w \sum_{j} \eta_j z_j = 1
\quad\Longleftrightarrow\quad
w = \frac{1}{D(z)},
\]
so \eqref{eq:cc_def} holds.

Under \eqref{eq:cc_def},
\[
\frac{N(z)}{D(z)}
=
\left(\sum_{j} \delta_j z_j\right) \cdot w
=
\sum_{j} \delta_j (w z_j)
=
\sum_{j} \delta_j y_j.
\]
Hence the fractional constraint \eqref{eq:pdr_frac} is equivalent to the pair of
linear constraints
\begin{equation}
\label{eq:pdr_lin}
\pdr_{\mathrm{lb}}
\;\le\;
\sum_{j} \delta_j y_j
\;\le\;
\pdr_{\mathrm{ub}},
\end{equation}
together with \eqref{eq:cc_norm}--\eqref{eq:rlt3}. This reformulation is exact,
because any feasible $z$ with $D(z)\ge 1$ can be extended to $(w,y)$ by setting
$w = 1/D(z)$ and $y_j = w z_j$, and conversely any feasible $(z,w,y)$ satisfies
\eqref{eq:cc_def} and therefore induces the same PDR value as in
\eqref{eq:pdr_frac}. This yields an exact MILP representation.
\endproof

\section{Evaluation Methodology} 

\subsection{Synthetic Control Design and Identification}\label{ec_SCM} 
We use the synthetic control method (SCM) to construct a counterfactual outcome
trajectory for each treated market as a weighted combination of control markets.
Let $s \in \{1,\dots,S\}$ index time periods (e.g., weeks). Let
$Y_s^{\text{T}}$ denote the observed outcome (e.g., upsell revenue per booking) of
the treated market at time $s$, and let $Y_s^{(j)}$ denote the outcome of control
market $j \in \{1,\dots,J\}$ at time $s$. SCM defines the synthetic control as
\begin{equation}
Y_s^{\text{SC}} = \sum_{j=1}^{J} w_j\, Y_s^{(j)},
\label{eq:scm_synth}
\end{equation}
where the weights $w_j$ form a convex combination:
\begin{equation}
w_j \ge 0 \ \ \forall j,
\qquad
\sum_{j=1}^{J} w_j = 1.
\label{eq:scm_simplex}
\end{equation}

Let $\text{Pre} \subset \{1,\dots,S\}$ denote the pre-intervention periods.
The SCM weights are chosen to best match the treated market's pre-intervention
trajectory by solving the constrained least-squares problem
\begin{equation}
\min_{w \in \mathbb{R}^J}
\ \sum_{s \in \text{Pre}}
\left(
Y_s^{\text{T}} - \sum_{j=1}^{J} w_j Y_s^{(j)}
\right)^2
\quad
\text{s.t. } \eqref{eq:scm_simplex}.
\label{eq:scm_opt}
\end{equation}

Given the fitted weights $w^\star$, the treatment effect at time $s$ is estimated
as the gap between the treated outcome and its synthetic counterpart:
\begin{equation}
\tau_s = Y_s^{\text{T}} - Y_s^{\text{SC}}.
\label{eq:scm_tau}
\end{equation}
Let $\text{Pilot}$ denote the set of post-intervention (pilot) periods with
$S_{\text{pilot}} := |\text{Pilot}|$. The average treatment effect over the pilot
window is
\begin{equation}
\bar{\tau}
=
\frac{1}{S_{\text{pilot}}}
\sum_{s \in \text{Pilot}} \tau_s.
\label{eq:scm_ate}
\end{equation}

\subsection{Robustness via Placebo-Based Inference}\label{appendix_placebo}

To further assess the robustness of the estimated treatment effect, we conduct a placebo test using SCM. This approach evaluates whether the observed effect in the treated group could plausibly have arisen by chance, and it is particularly valuable in observational field settings where randomized experiments are infeasible.

Recall that \( Y_s^{\text{T}} \) denote the observed outcome for the treatment group at time \( s \), and let \( Y_s^{\text{SC}} \) denote the outcome of its corresponding synthetic control, as constructed in Section~\ref{sect:scm}. Once the pilot begins, the treatment effect at each time \( s \) is estimated as:
\[
\tau_s = Y_s^{\text{T}} - Y_s^{\text{SC}}.
\]

\begin{figure}
 \centering
  \includegraphics[width=0.7\linewidth]
  {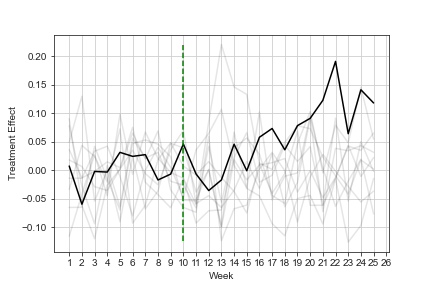}
  \caption{ Comparison of the treatment group effect under \textsc{COAT} (black line) with placebo effects estimated for each control market (gray lines).}
  \label{fig:placebo_test}
\end{figure}

For placebo analysis, we iteratively apply the SCM procedure to each control market \( j \in \mathcal{C} \), treating it as if it were exposed to the intervention. Let \( Y_s^{(j)} \) denote the observed outcome of control market \( j \), and let \( Y_s^{(j,\text{SC})} \) denote the synthetic control constructed for that placebo unit using the remaining markets as donors. The placebo effect (or gap) for unit \( j \) at time \( s \) is then:
\[
\tau_s^{(j)} = Y_s^{(j)} - Y_s^{(j,\text{SC})}.
\]
Each placebo synthetic control is refit independently with the same optimizer and predictor specification used for the treated unit, so that the placebo distribution reflects the same estimation procedure. Our donor pool comprises the \( N - 1 = 30 \) control markets, yielding 30 placebo trajectories in addition to the treated unit, for \( N = 31 \) units in total.

Figure~\ref{fig:placebo_test} displays the actual treatment effect trajectory \( \tau_s \) (black line), along with placebo trajectories \( \tau_s^{(j)} \) (gray lines) over the full evaluation window. During the pre-intervention period (Weeks 1–9), the treated and synthetic trajectories align closely, and the estimated treatment effect \( \tau_s \) lies within the range of placebo effects, validating the pre-treatment fit and justifying the counterfactual construction.

Following the intervention (Week 10 onward), \( \tau_s \) diverges sharply from the distribution of placebo trajectories, consistently lying in the upper tail. This provides strong evidence that the observed uplift is unlikely to result from noise or latent confounding.

\paragraph{Inference via the RMSPE ratio.}
To convert this visual evidence into a formal hypothesis test of the sharp null of no treatment effect, we follow \citet{abadie2010synthetic} and rank units by the ratio of post- to pre-intervention root mean squared prediction error (RMSPE). For each unit \( j \), define
\[
r_j \;=\; \frac{\dfrac{1}{T - T_0}\displaystyle\sum_{s > T_0}\big(\tau_s^{(j)}\big)^2}{\dfrac{1}{T_0}\displaystyle\sum_{s \le T_0}\big(\tau_s^{(j)}\big)^2},
\]
where \( T_0 \) indexes the last pre-intervention period (Week 9) and \( T \) the end of the evaluation window. Normalizing the post-intervention gap by the pre-intervention fit is essential: a placebo unit may exhibit a large post-period gap simply because its synthetic control fits poorly even before any intervention, and dividing by the pre-period MSPE penalizes such poorly-fit units rather than rewarding them. The associated permutation \( p \)-value is the share of units whose ratio is at least as large as the treated unit's,
\[
p \;=\; \frac{\#\{\, j : r_j \ge r_{\text{T}} \,\}}{N}.
\]

Empirically, the treated unit attains the \emph{largest} post/pre RMSPE ratio among all 31 units (\( r_{\text{T}} \approx 11.0 \)), exceeding every placebo; the next-largest ratio is \( 10.4 \), after which the placebo distribution drops sharply (the third-largest ratio is \( 3.2 \)). With the treated unit ranked first out of 31, the resulting \( p \)-value is \( 1/31 \approx 0.032 \), rejecting the sharp null of no effect at the 5\% level. This is consistent with the post-pilot average treatment effect of 6.9\% reported earlier, which lies well outside the placebo distribution.

These results reinforce the causal interpretation of the AI pricing policy's impact. Placebo SCM serves as a useful nonparametric inference method, offering internal validity checks without relying on parametric assumptions or randomized experimentation. We note that the resolution of the \( p \)-value is bounded below by \( 1/N \): with 31 units the smallest attainable value is \( 0.032 \), so the treated unit's first-place ranking already delivers the strongest rejection the donor pool can support.







\end{document}